\documentclass{article}
\usepackage{ijcai23}

\usepackage{times}
\usepackage{soul}
\usepackage{url}
\usepackage[hidelinks]{hyperref}
\usepackage[utf8]{inputenc}
\usepackage[small]{caption}
\usepackage{graphicx}
\usepackage{amsmath}
\usepackage{amsthm}
\usepackage{booktabs}
\usepackage{algorithm}
\usepackage[switch]{lineno}

\usepackage{amssymb}
\usepackage{amsfonts}
\usepackage{mathtools}
\usepackage{xcolor}
\usepackage[noend]{algpseudocode}
\usepackage{centernot}
\usepackage{makecell}
\usepackage{enumerate}
\usepackage[makeroom]{cancel}
\usepackage{yhmath}
\usepackage{import}
\usepackage[makeroom]{cancel}
\usepackage{multirow}
\usepackage{wrapfig}

\usepackage{subcaption}

\newtheorem{theorem}{Theorem}
\newtheorem{corollary}{Corollary}
\newtheorem{lemma}{Lemma}
\newtheorem{definition}{Definition}

\newcommand{\vars}{\mathtt{vars}}

\def\X{{\mathbf X}}
\def\Y{{\mathbf Y}}

\def\S{{\mathbf S}}
\def\R{{\mathbf R}}
\def\C{{\mathbf C}}
\def\Q{{\mathbf Q}}
\def\x{{\mathbf x}}

\newcommand{\dup}[1]{[#1]}
\newcommand{\ndup}[2]{{#1}^{#2}}

\def\cls{{\mathbf C}}

\newcommand\pr{{\it Pr}}
\newcommand\n[1]{{\bar{#1}}}
\def\jtree{{\cal T}}
\def\hosts{{\cal H}}

\def\JT{{\langle\jtree,\hosts\rangle}}
\def\TJT{{\langle\jtree^t,\hosts^t\rangle}}

\usepackage{tikz}
\tikzset{state/.style = {shape=circle,draw,thick,minimum size=3.0em}}
\tikzset{dstate/.style = {shape=circle,draw,thick,double,minimum size=3.0em}}
\tikzset{point/.style = {circle, draw, thick, inner sep=0.05cm,fill,node contents={}}}
\usetikzlibrary{shapes,decorations,calc,fit,positioning}

\title{On the Complexity of Counterfactual Reasoning}

\author{
Yunqiu Han
\and
Yizuo Chen
\and
Adnan Darwiche
\affiliations
University of California, Los Angeles
\emails
yunqiu21@g.ucla.edu,
yizuo.chen@ucla.edu,
darwiche@cs.ucla.edu
}

\begin{document}

\maketitle

\begin{abstract}
We study the computational complexity of counterfactual reasoning in relation to the complexity of associational and interventional reasoning on
structural causal models (SCMs). We show that counterfactual reasoning is no harder than associational or interventional reasoning 
on fully specified SCMs in the context of two computational frameworks. 
The first framework is based on the notion of {\em treewidth}
and includes the classical variable elimination and jointree algorithms. The second framework is based on the more recent and refined
notion of {\em causal treewidth} which is directed towards models 
with functional dependencies such as SCMs. Our results are constructive
and based on bounding the (causal) treewidth of twin networks---
used in standard counterfactual reasoning that contemplates 
two worlds, real and imaginary---to the 
(causal) treewidth of the underlying SCM structure. In particular,
we show
that the latter (causal) treewidth is no more than twice the former plus one. Hence,
if associational or interventional reasoning is tractable on a fully
specified SCM
then counterfactual reasoning is tractable too. We extend our results to general counterfactual reasoning that requires contemplating more than two
worlds and discuss applications of our results to counterfactual reasoning
with partially specified SCMs that are coupled with data. 
We finally present empirical 
results that measure the gap between the complexities of counterfactual
reasoning and associational/interventional reasoning on random SCMs.
\end{abstract}

\section{Introduction}

A theory of causality has emerged over the last few decades based on two parallel hierarchies, 
an {\em information hierarchy} and a {\em reasoning hierarchy,} often called the {\em causal hierarchy}~\cite{pearl18}. 
On the reasoning side, this theory has crystallized three levels of reasoning with increased sophistication and proximity 
to human reasoning: associational, interventional and counterfactual, which are exemplified by the following canonical probabilities.
{\em Associational} \(\pr(y | x)\): probability of \(y\) given that \(x\) was observed. Example: probability that a patient has a flu given they have a fever.
{\em Interventional} \(\pr(y_x)\): probability of \(y\) given that \(x\) was established by an intervention. This is different from \(\pr(y | x)\). 
Example: seeing the barometer fall tells us about the weather but moving the barometer needle won't bring rain.
{\em Counterfactual} \(\pr(y_x | \n{y}, \n{x})\): probability of \(y\) if we were to establish \(x\) by an intervention
 given that neither \(x\) nor \(y\) are true. Example: probability that a patient who
did not take a vaccine and died would have recovered had they been vaccinated.
On the information side, these forms of reasoning were shown to require different
levels of knowledge, encoded as (1)~associational models, (2)~causal models and (3)~functional (mechanistic) models, respectively,
with each class of models containing more information than the preceding one.
In the framework of probabilistic graphical models~\cite{kollerbook}, this information is encoded by (1)~Bayesian networks~\cite{DarwicheBook09,Pearl88b}, 
(2)~causal Bayesian networks~\cite{pearl00b,PetersBook,SpirtesBook}, and (3)~functional Bayesian networks~\cite{uai/BalkeP95,pearl00b}.

\begin{figure}[tb]
\begin{center}
 \includegraphics[width=0.98\linewidth]{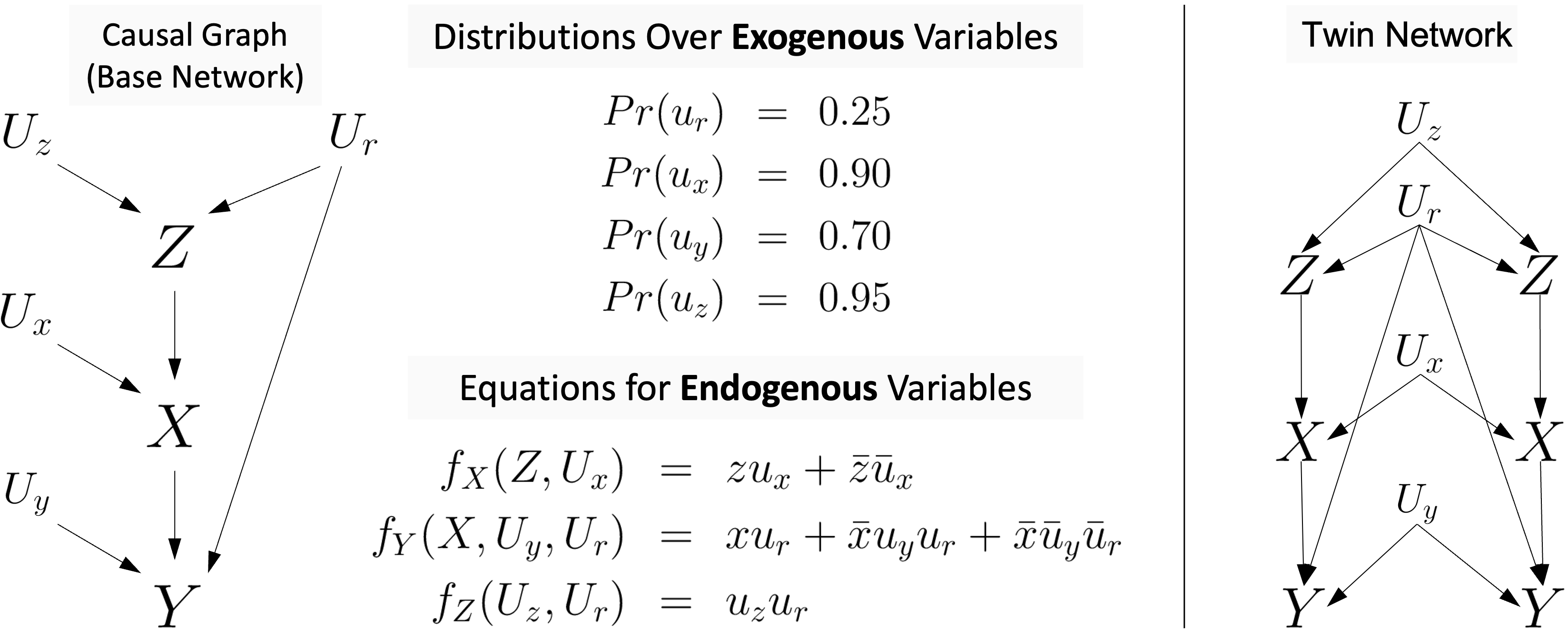}
\end{center}
\caption{\small A structural causal model~{\protect\cite{Bareinboim20211OP}} and its twin network.
Endogenous variables represent treatment (\(X\)), the outcome of (\(Y\)), and the presence of (\(Z\)), hypertension. 
Exogenous variables represent natural resistance to disease (\(U_r\)) and 
sources of variation affecting endogenous variables 
(\(U_x, U_y, U_z\)).}
\label{fig:SCM}
\end{figure}

Counterfactual reasoning has received much interest as it inspires both introspection and contemplating scenarios that 
have not been seen before, and is therefore viewed by many as a hallmark of human intelligence.
Figure~\ref{fig:SCM} depicts a functional Bayesian network, also known as a {\em structural causal model (SCM)}~\cite{galles1998axiomatic,halpern2000axiomatizing}, which can be used
to answer counterfactual queries. 
Variables without causes are called {\em exogenous} or {\em root} and variables with causes are called {\em endogenous} or {\em internal.} 
The only uncertainty in SCMs concerns the states of exogenous variables and this uncertainty is quantified using distributions over these variables. 
Endogenous variables are assumed to be {\em functional:}
they are functionally determined by their causes where the functional
relationships, also known as {\em causal mechanisms,} are specified by structural equations.\footnote{These equations can also be specified using
conditional probability tables (CPTs)  that are normally used in Bayesian networks, but the CPTs will contain only deterministic distributions.}
These equations and the distributions over exogenous variables define the {\em parameters}
of the causal graph, leading to a fully specified SCM which can be used to evaluate associational, interventional and
counterfactual queries. For example, the SCM in Figure~\ref{fig:SCM} has enough information to evaluate the counterfactual query 
\(\pr(y_x | \n x, \n y)\): the probability that a patient who did not take the treatment and died would have been alive had they been given the treatment (\(2.17\%\)). A causal Bayesian network contains less information than a functional one (SCM) as it does not require endogenous variables to be functional, 
but it is sufficient to compute associational and interventional probabilities. 
A Bayesian network contains even less information as it does not require network edges to have 
a causal interpretation, only that the conditional independences encoded by the network are correct, so
it can only compute associational probabilities.

All three forms of reasoning (and models) involve a directed acyclic graph (DAG) which we call the {\em base network;} see left of Figure~\ref{fig:SCM}.
Associational and interventional reasoning can be implemented by applying
classical inference algorithms to the base network. The time
complexity can be bounded by \(n \cdot \exp(w)\), where \(n\) is the
number of network nodes and \(w\) is its treewidth (a graph-theoretic measure of connectivity).
Counterfactual reasoning is more sophisticated and is based on a three-step process that involves abduction, intervention and prediction~\cite{aaai/BalkeP94}. When contemplating two worlds, 
this process can be implemented by applying classical inference algorithms
to a {\em twin network}~\cite{aaai/BalkeP94}, obtained by duplicating endogenous nodes in the base network; 
see right of Figure~\ref{fig:SCM}. To compute the counterfactual query \(\pr(y_x | \n{y}, \n{x})\), one asserts \(\n{y}, \n{x}\) as an observation
on one side of the twin network (real world) and computes the interventional query \(\pr(y_x)\) on the other side of the network (imaginary world).
The time complexity can be bounded by \(n^t \cdot \exp(w^t)\), where \(n^t\) is the number of nodes in
the twin network and \(w^t\) is its treewidth. 
A recent result provides a much tighter bound using the notion of {\em causal treewidth}~\cite{uai/ChenDarwiche22,causalityAC}, which
is no greater than treewidth but applies only when certain nodes in the base network are functional --- in SCMs every endogenous node 
is functional.

One would expect the more sophisticated counterfactual reasoning with twin networks to be more expensive than associational/interventional reasoning 
with base networks since the former networks are larger and have more complex topologies. But the question is: How much more expensive?
For example, can counterfactual reasoning be intractable
on a twin network when associational/interventional reasoning 
is tractable on its base network?
We address this question in the context of reasoning algorithms whose complexity is exponential only in the (causal) treewidth,
such as the jointree algorithm~\cite{jointreeAlg}, the variable elimination algorithm~\cite{uai/ZhangP94,dechterUAI96} and circuit-based algorithms~\cite{DarwicheJACM03,neusys22}.
In particular, we show in Sections~\ref{sec:tw}~\&~\ref{sec:ctw}
that the (causal) treewidth of a twin network is at most twice the (causal) treewidth of its base network plus one. Hence, the complexity of counterfactual reasoning on fully specified SCMs is no more than quadratic in the 
complexity of associational and interventional reasoning, so the former must be tractable if the latter is tractable. We extend our results in Section~\ref{sec:n-worlds} to counterfactual reasoning 
that requires contemplating more than two worlds, where we also
discuss a class 
of applications that require this type of reasoning and for
which fully specified SCMs can be readily available. Our
results apply directly to counterfactual reasoning on fully specified SCMs but
we also discuss in Section~\ref{sec:pscms} how they can sometimes be used in
the context of  
counterfactual reasoning on data and a partially specified SCM.
We finally present empirical results in Section~\ref{sec:expS} which reveal that, on average,
the complexity gap between counterfactual and associational/interventional reasoning on fully specified SCMs can be smaller than what our worst-case bounds may suggest.

\section{Technical Preliminaries}
\label{sec:prelim}

We next review the notions of treewidth~\cite{jal/RobertsonS86} and 
causal treewidth~\cite{uai/ChenDarwiche22,causalityAC,DarwicheECAI20b} which we use to characterize the 
computational complexity of counterfactual reasoning on fully specified SCMs.
We also review the notions of elimination orders, jointrees and thinned jointrees which are the basis for defining (causal) treewidth
and act as data structures that characterize the computational
complexity of various reasoning algorithms.
We use these notions extensively when stating and proving our results (proofs of all results are in Appendix~\ref{app:proof}).
We assume all variables are discrete.
A variable is denoted by an uppercase letter (e.g. $X$) and its 
values by a lowercase letter (e.g. $x$). 
A set of variables is denoted by a bold uppercase letter (e.g. $\X$) and its instantiations by a bold lowercase letter (e.g. $\x$).

\subsection{Elimination Orders and Treewidth}
These are total orders of the network variables which drive,
and characterize the complexity of, the classical variable elimination
algorithm when computing associational, interventional and counterfactual queries.
Consider a DAG $G$ where every node represents a variable. 
An \emph{elimination order} $\pi$ for $G$ is a total ordering of the variables in $G$, where $\pi(i)$ is the $i^{th}$ variable in the order, starting from $i=1$. 
An elimination order defines an elimination process on the moral graph of DAG \(G\) which is used to define the treewidth of \(G\).
The \emph{moral graph} $G_m$ is obtained from $G$ by adding an undirected edge between every pair of common parents and 
then removing directions from all directed edges. 
When we \emph{eliminate} variable $\pi(i)$ from $G_m$, we connect every pair of neighbors of $\pi(i)$ in $G_m$ 
and remove $\pi(i)$ from $G_m$. 
This elimination process induces a \emph{cluster sequence} $\cls_1,\cls_2,\dots,\cls_n$, where $\cls_i$ is $\pi(i)$ and its neighbors in 
$G_m$ just before eliminating \(\pi(i)\). The \emph{width} of an elimination order is the size of its largest induced cluster minus~\(1\). 
The \emph{treewidth} for DAG $G$ is the minimum width of any elimination order for $G$.
The variable elimination algorithm computes
queries in \(O(n \cdot \exp(w))\) time where \(n\) is the number of nodes
in the (base or twin) network and \(w\) is the width of a corresponding elimination 
order.
Elimination orders are usually constructed using heuristics that aim to minimize their width.
We use the popular {\em minfill} heuristic~\cite{ELMheuristics} in our experiments 
while noting that more effective heuristics may exist
as shown in~\cite{Kjaerulff1994,Larranaga1997}.

\begin{figure}[tb]
\centering
\begin{subfigure}[b]{0.18\columnwidth}
\centering
\begin{tikzpicture}[->,>=stealth,shorten >=1pt,auto,scale=0.36,transform shape]
\node[state,thin,font=\huge] (B) at (-1.5,-1) {$A$};
\node[state,thin,font=\huge] (C) at (1.5,-1) {$B$};
\node[state,thin,font=\huge] (D) at (-1.5,-3) {$D$};
\node[state,thin,font=\huge] (E) at (0,-2) {$C$};
\node[state,thin,font=\huge] (G) at (0,-4.5) {$F$};
\node[state,thin,font=\huge] (F) at (1.5,-3) {$E$};
\path (B) edge (E);
\path (C) edge (E);
\path (B) edge (D);
\path (D) edge (G);
\path (E) edge (G);
\path (D) edge (F);
\path (C) edge (F);
\end{tikzpicture}
\caption{DAG}
\label{fig:dag-rep}
\end{subfigure}
\hfill
\begin{subfigure}[b]{0.34\columnwidth}
\centering
\includegraphics[width=.95\linewidth]{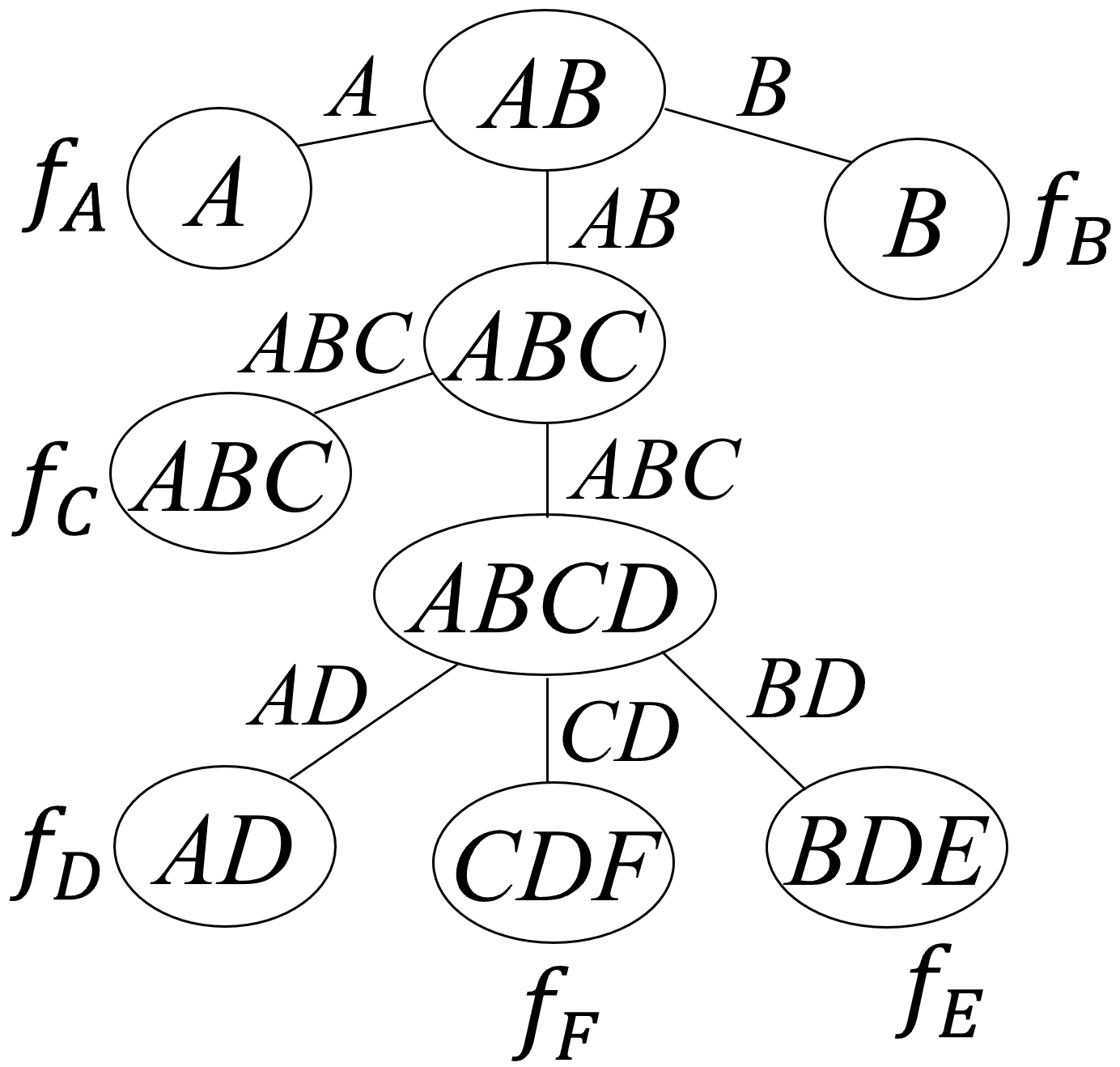}
\caption{jointree (width \(3\))}
\label{fig:jt1}
\end{subfigure}
\hfill
\begin{subfigure}[b]{0.44\columnwidth}
\centering
\includegraphics[width=.95\linewidth]{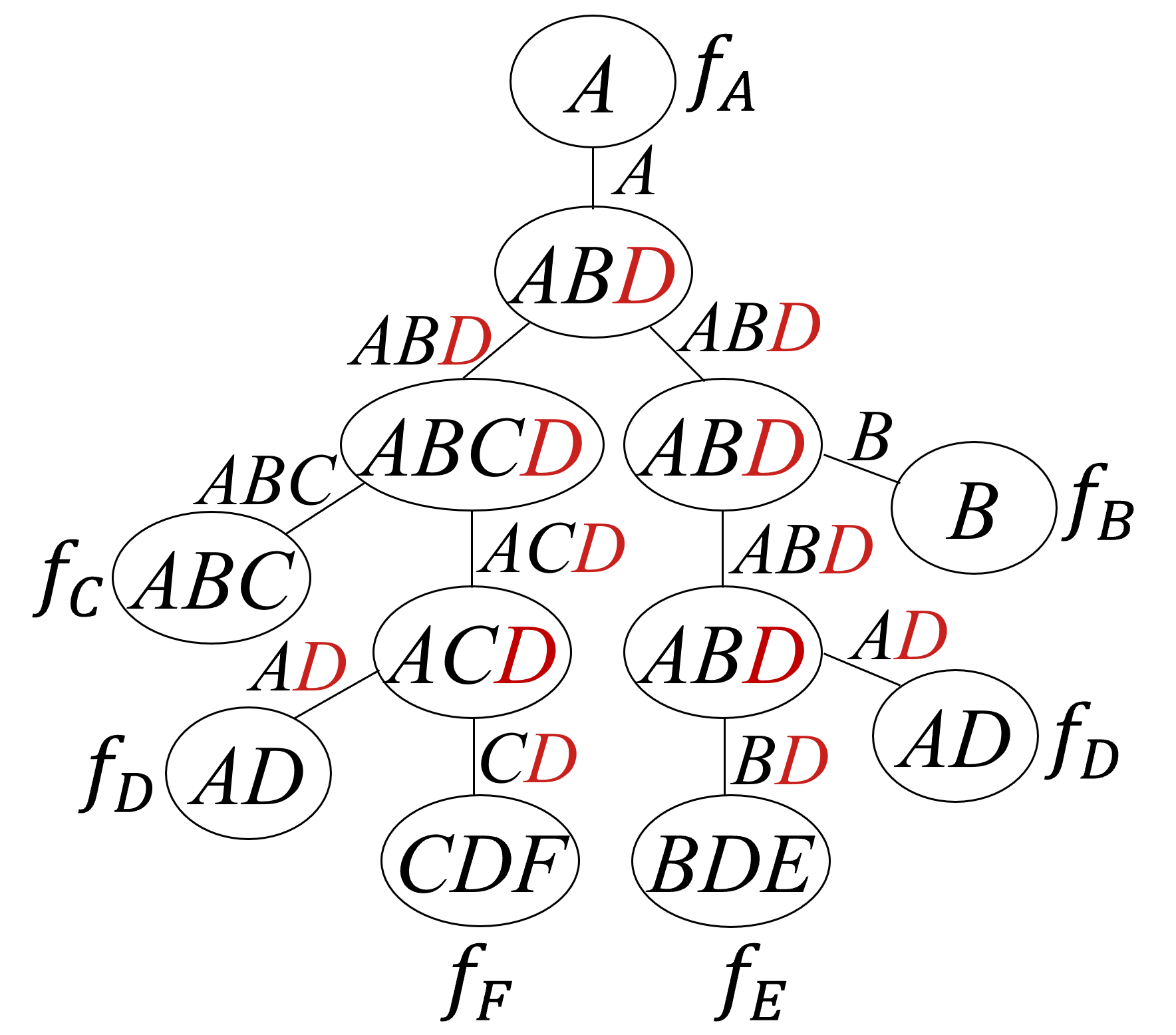}
\caption{thinned jointree (width \(2\))}
\label{fig:tjt}
\end{subfigure}
\hfill
\caption{A family \(f\) appears next to a jointree 
node \(i\) iff \(f\) is hosted by \(i\) (\(i \in \hosts(f)\)).
$D$ is functional and red variables are thinned.
\label{fig:ex1}}
\end{figure}
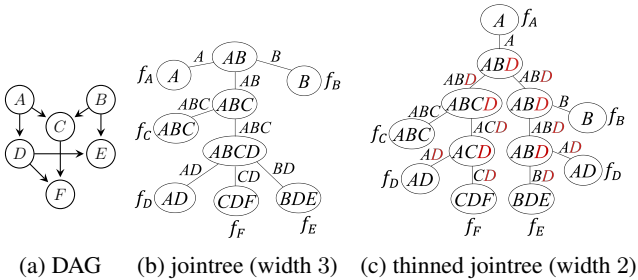
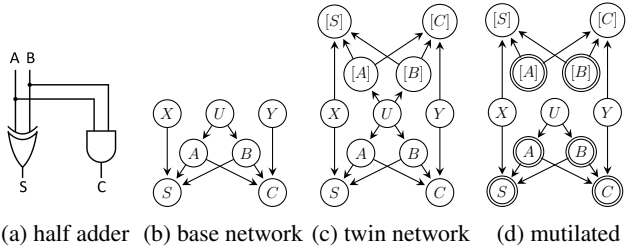
\begin{figure}[tb]
\begin{center}
\begin{subfigure}[b]{0.22\columnwidth}
\centering
\includegraphics[width=\linewidth]{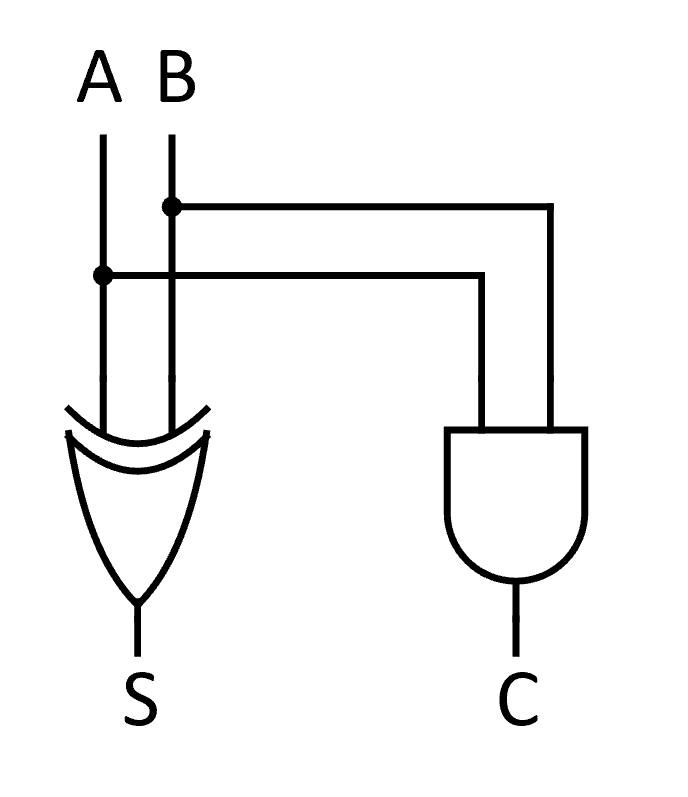}
\caption{half adder}
\label{fig:half-adder}
\end{subfigure}
\hfill
\begin{subfigure}[b]{0.25\columnwidth}
\centering
\begin{tikzpicture}[->,>=stealth,shorten >=1pt,auto,scale=0.35,transform shape]
\node[state,thin,font=\huge] (U) at (0,0) {$U$};
\node[state,thin,font=\huge] (X) at (-2,0) {$X$};
\node[state,thin,font=\huge] (Y) at (2,0) {$Y$};
\node[state,thin,font=\huge] (A) at (-1,-1.5) {$A$};
\node[state,thin,font=\huge] (B) at (1,-1.5) {$B$};
\node[state,thin,font=\huge] (S) at (-2,-3) {$S$};
\node[state,thin,font=\huge] (C) at (2,-3) {$C$};
\path (U) edge (A);
\path (U) edge (B);
\path (A) edge (S);
\path (B) edge (S);
\path (X) edge (S);
\path (A) edge (C);
\path (B) edge (C);
\path (Y) edge (C);
\end{tikzpicture}
\caption{base network}
\label{fig:base-net}
\end{subfigure}
\hfill
\begin{subfigure}[b]{0.25\columnwidth}
\centering
\begin{tikzpicture}[->,>=stealth,shorten >=1pt,auto,scale=0.35,transform shape]
\node[state,thin,font=\huge] (U) at (0,0) {$U$};
\node[state,thin,font=\huge] (X) at (-2,0) {$X$};
\node[state,thin,font=\huge] (Y) at (2,0) {$Y$};
\node[state,thin,font=\huge] (A) at (-1,-1.5) {$A$};
\node[state,thin,font=\huge] (B) at (1,-1.5) {$B$};
\node[state,thin,font=\huge] (S) at (-2,-3) {$S$};
\node[state,thin,font=\huge] (C) at (2,-3) {$C$};

\node[state,thin,font=\huge] (A1) at (-1,1.5) {$\dup{A}$};
\node[state,thin,font=\huge] (B1) at (1,1.5) {$\dup{B}$};
\node[state,thin,font=\huge] (S1) at (-2,3.5) {$\dup{S}$};
\node[state,thin,font=\huge] (C1) at (2,3.5) {$\dup{C}$};

\path (U) edge (A);
\path (U) edge (B);
\path (A) edge (S);
\path (B) edge (S);
\path (X) edge (S);
\path (A) edge (C);
\path (B) edge (C);
\path (Y) edge (C);

\path (U) edge (A1);
\path (U) edge (B1);
\path (A1) edge (S1);
\path (B1) edge (S1);
\path (X) edge (S1);
\path (A1) edge (C1);
\path (B1) edge (C1);
\path (Y) edge (C1);

\end{tikzpicture}
\caption{twin network}
\label{fig:twin-net}
\end{subfigure}
\hfill
\begin{subfigure}[b]{0.25\columnwidth}
\centering
\begin{tikzpicture}[->,>=stealth,shorten >=1pt,auto,scale=0.35,transform shape]
\node[state,thin,font=\huge] (U) at (0,0) {$U$};
\node[state,thin,font=\huge] (X) at (-2,0) {$X$};
\node[state,thin,font=\huge] (Y) at (2,0) {$Y$};
\node[dstate,thin,font=\huge] (A) at (-1,-1.5) {$A$};
\node[dstate,thin,font=\huge] (B) at (1,-1.5) {$B$};
\node[dstate,thin,font=\huge] (S) at (-2,-3) {$S$};
\node[dstate,thin,font=\huge] (C) at (2,-3) {$C$};

\node[dstate,thin,font=\huge] (A1) at (-1,1.5) {$\dup{A}$};
\node[dstate,thin,font=\huge] (B1) at (1,1.5) {$\dup{B}$};
\node[state,thin,font=\huge] (S1) at (-2,3.5) {$\dup{S}$};
\node[state,thin,font=\huge] (C1) at (2,3.5) {$\dup{C}$};

\path (U) edge (A);
\path (U) edge (B);
\path (A) edge (S);
\path (B) edge (S);
\path (X) edge (S);
\path (A) edge (C);
\path (B) edge (C);
\path (Y) edge (C);

\path (A1) edge (S1);
\path (B1) edge (S1);
\path (X) edge (S1);
\path (A1) edge (C1);
\path (B1) edge (C1);
\path (Y) edge (C1);

\end{tikzpicture}
\caption{mutilated}
\label{fig:twin-net2}
\end{subfigure}
\end{center}
\hfill
\caption{Internal nodes in the base network (Figure (b)) are functional. Double-circled nodes have evidence. }
\end{figure}

\subsection{Jointrees and Treewidth}
These are data structures that drive, and characterize the
complexity of, the classical jointree algorithm; see Figure~\ref{fig:jt1}.
Let the {\em family} of variable \(X\) in DAG \(G\) be the set \(f_X\) containing \(X\) and its parents in \(G\).
A \emph{jointree} for DAG $G$ is a pair \(\JT\) where \(\jtree\) is a tree and \(\hosts\) is a function that maps each family \(f\) of \(G\) into nodes \(\hosts(f)\) in \(\jtree\) called the {\em hosts} of family \(f\). 
The requirements are: only nodes with a single neighbor (called {\em leaves}) can be hosts; each leaf node hosts exactly one family;
and each family must be hosted by at least one node.\footnote{The 
standard definition of jointrees allows any node to be a host of 
any number of families. Our definition facilitates the
upcoming treatment and does not preclude optimal jointrees.}
This induces a {\em cluster} \(\C_i\) for each jointree node \(i\)
and a {\em separator} \(\S_{ij}\) for each jointree edge \((i,j)\) which are defined as follows.
Separator \(\S_{ij}\) is the set of variables hosted at both sides of edge \((i,j)\).
If jointree node \(i\) is a leaf then cluster \(\C_i\) is the family hosted by \(i\); 
otherwise, \(\C_i\) is the union of separators adjacent to node \(i\). 
The \emph{width} of a jointree is the size of its largest cluster minus~\(1\). 
The minimum width attained by any jointree for $G$ corresponds to the treewidth of \(G\).
The jointree algorithm computes queries in \(O(n \cdot \exp(w))\) 
time where \(n\) is the number of nodes 
and \(w\) is the width of a corresponding jointree.
Jointrees are usually constructed from elimination orders, and
there are polytime, width-preserving transformations between elimination orders and jointrees; see~\cite[Ch~9]{DarwicheBook09} for details.

\subsection{Thinned Jointrees and Causal Treewidth} 

To {\em thin} a jointree is to remove some variables from its separators (and hence clusters, which are defined in terms of separators);
see Figure~\ref{fig:tjt}.
Thinning can reduce the jointree width quite significantly, leading to exponential savings in reasoning time. 
Thinning is possible only when some variables in the network are 
functional, even without knowing the specific functional 
relationships (i.e., structural equations).
The {\em causal treewidth} is the minimum width for any thinned jointree.
Causal treewidth is no greater than treewidth and the former can be bounded when the latter is not. 
While this notion can be applied broadly as in~\cite{DarwicheECAI20b}, it is particularly relevant to counterfactual reasoning since every internal node in an SCM is functional so the causal treewidth for SCMs
can be much smaller
than their treewidth.
There are alternate definitions of thinned jointrees. The next
definition is based on thinning rules~\cite{uai/ChenDarwiche22}.

A thinning rule removes a variable from a separator under certain conditions.
There are two thinning rules which apply only to functional variables. 
The first rule removes variable \(X\) from a separator
\(\S_{ij}\) if edge \((i,j)\) is on the path between two leaf nodes 
that host the family of \(X\) and every separator on that path contains \(X\). 
The second rule removes variable \(X\) from a separator \(\S_{ij}\) if no other separator \(\S_{ik}\) contains \(X\), or no other separator \(\S_{kj}\)
contains \(X\). A thinned jointree is obtained by applying these rules to
exhaustion. Figure~\ref{fig:ex1} depicts an optimal, classical jointree and a thinned jointree for the same DAG (the latter has smaller width).

The effectiveness of thinning rules depends on the number of jointree
nodes that host a family \(f\), \(|\hosts(f)|\), and the location
of these nodes in the jointree. One can enable more thinnings by 
increasing the number of jointree nodes that host each family \(f\).
This process is called {\em replication} where \(|\hosts(f)|\) is
called the number of {\em replicas} for family \(f\). Replication
comes at the expense of increasing the number of jointree nodes
so the definition
of causal treewidth limits this growth by requiring the 
jointree size to be a polynomial in the number of nodes in the underlying DAG; see~\cite{uai/ChenDarwiche22} for details.\footnote{Thinning
rules will not trigger if families are not replicated (\(|\hosts(f)|=1\) 
for all \(f\)). Replication usually increases the width 
of a jointree from \(w\) to \(w_r\) with the goal of having thinning
rules reduce width \(w_r\) to width \(w_t < w \leq w_r\). The replication 
strategy may sometimes not be effective on certain networks, 
leading to \(w < w_t \leq w_r\). We use the strategy
in~\cite{uai/ChenDarwiche22}
which exhibits this behavior on certain networks.\label{foot:replication}}

\section{The Treewidth of Twin Networks}
\label{sec:tw}
Consider Figure~\ref{fig:half-adder} which depicts a 2-bit half adder. Suppose the binary inputs $A$ and $B$ are randomly sampled from some distribution and the gates may not be functioning properly. This circuit can be modeled using the network in Figure~\ref{fig:base-net}. Variables $A,B,S,C$ represent the inputs and outputs of the circuit; $X,Y$ represent the health of the XOR gate and the AND gate; and $U$ represents an unknown external random sampler that decides the state of inputs $A$ and $B$. 
Suppose that currently input $A$ is high, input $B$ is low, yet both
outputs $C$ and $S$ are low which is an abnormal circuit behavior. We wish to know whether the half adder would still behave correctly when we turn both inputs $A$ and $B$ on. This question can be formulated using the following counterfactual query: $Pr((c,\n{s})_{a,b}|a,\n{b},\n{c},\n{s})$. 
This query can be answered using a twin network as shown in Figure~\ref{fig:twin-net}, where each non-root variable $V$ has a 
duplicate $\dup{V}$. 
The current evidence \(a,\n{b},\n{c},\n{s}\)
is asserted on the variables \(A,B,C,S\) representing
the real world and the 
interventional query $Pr((c,\n{s})_{a,b})$ is computed 
on the duplicate variables \(\dup{A},\dup{B},\dup{C},\dup{S}\)
representing the imaginary world. This is done by removing
the edges incoming into the intervened upon variables 
$\dup{A},\dup{B}$, asserting evidence $\dup{a},\dup{b}$ and 
finally computing the probability of $\dup{c},\dup{\n{s}}$ 
as shown in Figure~\ref{fig:twin-net2}; see~\cite{pearl00b}
for an elaborate discussion of these steps.
This basically illustrates how a counterfactual query
can be computed using algorithms for associational
queries, like variable elimination,
but on a mutilated twin network instead of the base network.

We next show that the treewidth of a twin network is at most
twice the treewidth of its base network plus one,
which allows us to relate the complexities of assocational,
interventional and counterfactual reasoning on fully specified SCMs.
We first recall the definition of twin networks as proposed by~\cite{aaai/BalkeP94}. 

\begin{definition}
Given a base network $G$, its \underline{twin network} $G^t$ is constructed as follows.
For each internal variable $X$ in $G$, add a new variable labeled $\dup{X}$. For each parent $P$ of $X$, if $P$ is an internal variable, make $\dup{P}$ a parent of $\dup{X}$; otherwise, make $P$ a parent of $\dup{X}$. We will call $X$ a \underline{base variable} and $\dup{X}$ a \underline{duplicate variable}.
\end{definition}

For convenience, we use $\dup{U}=U$ when \(U\) is root.
For variables $\X$, we use $\dup{\X}$ 
to denote $\{\dup{X} | X \in \X\}$.
Figure~\ref{fig:twin-net} depicts the twin network for the base 
network in Figure~\ref{fig:base-net}. 

\subsection{Twin elimination orders}

Our result on the treewidth of twin networks is based
on converting every elimination order for the base network
into an elimination order for its twin network while
providing a guarantee on the width of the latter in terms
of the width of the former. We provide a
similar result for jointrees that we use when 
discussing the causal treewidth of twin networks.

\begin{definition}~\label{def:teo}
Consider an elimination order $\pi$ for a base network $G$.
The \underline{twin elimination order} $\pi^t$ is an elimination order for its twin network $G^t$ constructed by replacing each non-root variable $X$ in $\pi$ by $X,\dup{X}$.
\end{definition}

Consider the base network in Figure~\ref{fig:base-net} and
its elimination order $\pi = A$, $B$, $X$, $Y$, $S$, $C$, $U$.
The twin elimination order will be $\pi^t = A$, $\dup{A}$, $B$, $\dup{B}$, $X$, $Y$, $S$, $\dup{S}$, $C$, $\dup{C}$, $U$.
Recall that eliminating variables \(\pi(i), \ldots, \pi(n)\)
from a base network $G$ induces a cluster sequence \(\C_1, \ldots, \C_n\). 
We use \(\C(X)\) to denote the cluster of eliminated
variable \(X\). Similarly, eliminating variables from a twin network $G^t$ induces a cluster sequence and we use \(\C^t(X)\) to 
denote the cluster of eliminated variable $X$ and \(\C^t(\dup{X})\)
to denote the cluster of its eliminated duplicate $\dup{X}$. 

\begin{theorem}\label{thm:twin-cls}
Suppose we are eliminating variables from base network \(G\)
using an elimination order \(\pi\) and
eliminating variables from its twin network \(G^t\) using the
twin elimination order \(\pi^t\).
For every variable \(X\) in \(G\), we have
$\cls^t(X) \subseteq \cls(X) \cup \dup{\cls(X)}$ and $\cls^t(\dup{X}) \subseteq \cls(X) \cup \dup{\cls(X)}$.
\end{theorem}

This theorem has two key corollaries. The first
relates the widths of an elimination order and 
its twin elimination order.

\begin{corollary}
\label{cor:teo-width}
Let $w$ be the width of  elimination order $\pi$ 
for base network $G$ and let $w^t$ be the width of twin 
elimination order $\pi^t$ for twin network $G^t$. 
We then have $w^t \leq 2w+1$.
\end{corollary}
The above bound is tight as shown in Appendix~\ref{app:tight}.
The next corollary gives us our first major result.

\begin{corollary}
~\label{cor:twin-treewidth}
If \(w\) is the treewidth of base network $G$ and
\(w^t\) is the treewidth of its twin network $G^t$,
then $w^t \leq 2w+1$.
\end{corollary}

\subsection{Twin jointrees}
\begin{figure}[tb]
\centering
\begin{subfigure}[b]{0.32\columnwidth}
\centering
\includegraphics[width=1\columnwidth]{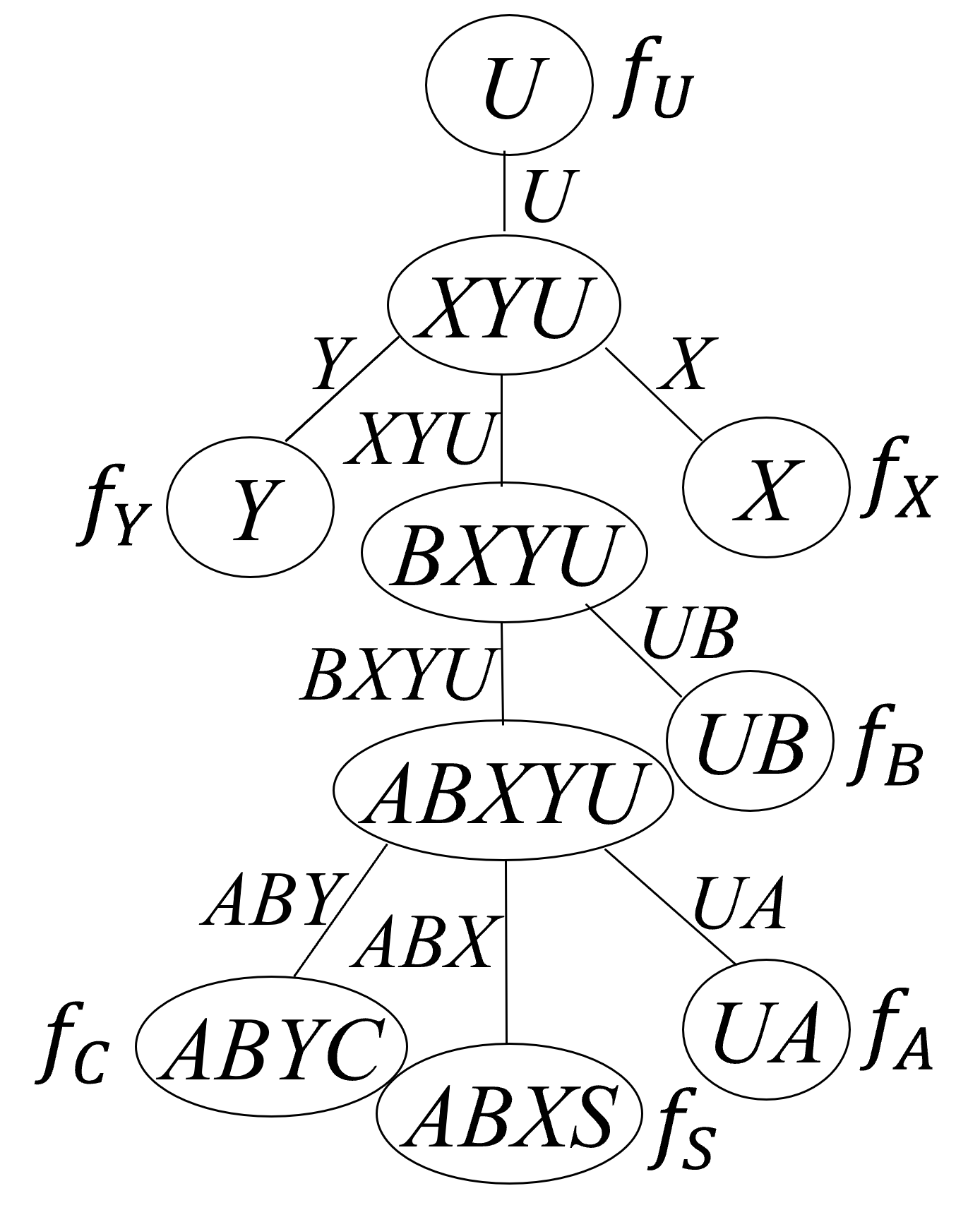}
\caption{base jointree}
\label{fig:base-jt}
\end{subfigure}
\hfill
\begin{subfigure}[b]{0.57\columnwidth}
\centering
\includegraphics[width=1\columnwidth]{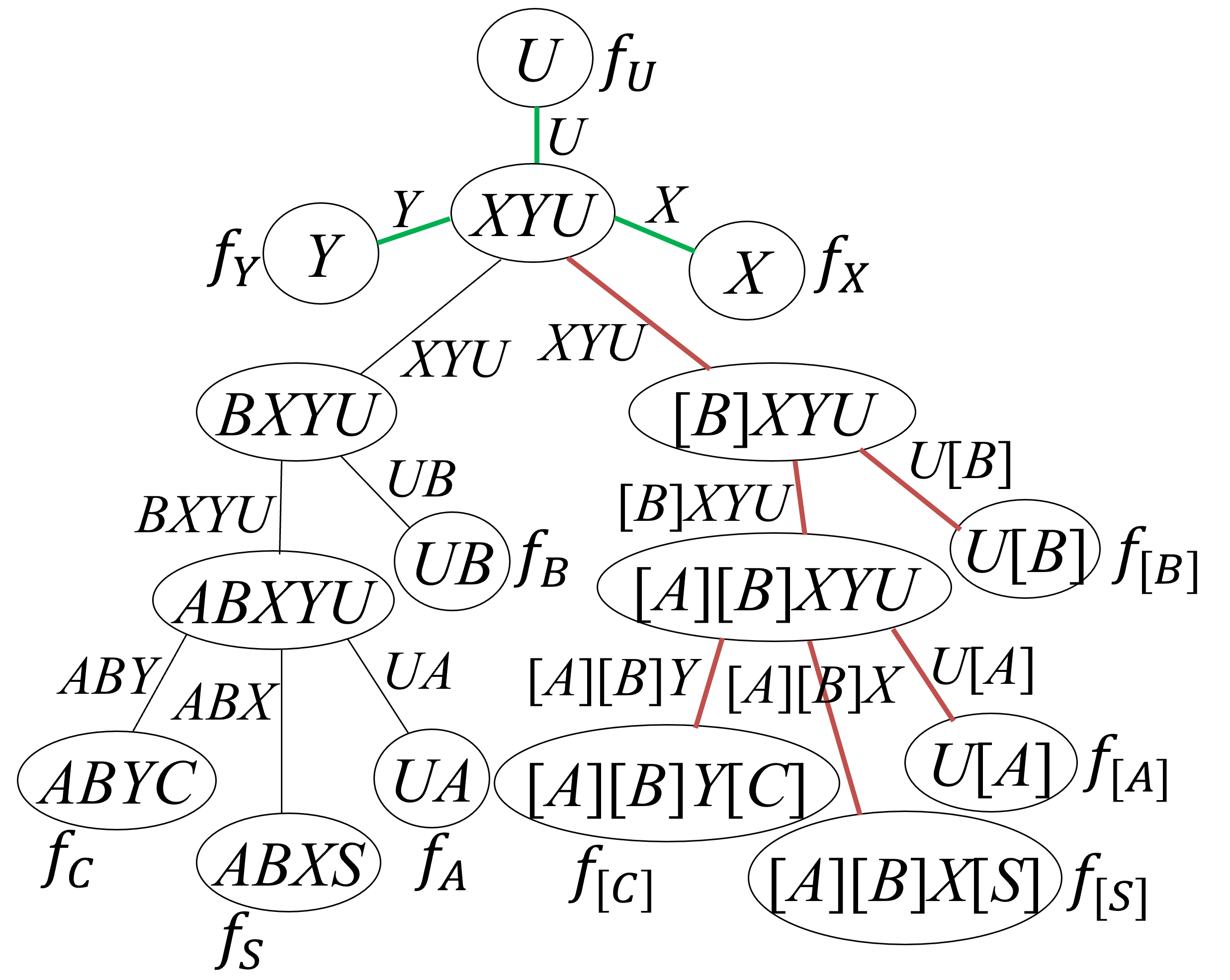}
\caption{twin jointree using Algorithm~\ref{alg:jt}}
\label{fig:twin-jt}
\end{subfigure}
\hfill
\caption{A family \(f\) appears next to a jointree node \(i\)
iff the family is hosted by that node (\(i \in \hosts(f)\)).
\label{fig:alg-ex}}
\end{figure}
We will now provide a similar result for jointrees. That is,
we will show how to convert a jointree \(\JT\)
for a base network \(G\) into a jointree \(\TJT\) for its twin network \(G^t\) while providing a guarantee
on the width/size of the twin jointree in terms of the width/size of 
the base jointree. This may seem like a redundant result given
Corollary~\ref{cor:teo-width} but the provided conversion 
will actually be critical for our later result on bounding the causal treewidth of twin networks. It can also be significantly more efficient 
than constructing a jointree by operating on the (larger) twin network.

Our conversion process operates on a {\em jointree} after 
directing its edges away from some node \(r\), call it
a {\em root.}
This defines a single parent for each jointree node \(i \neq r\), which is the neighbor of \(i\) closest to root \(r\), with all other neighbors of \(i\) being its children. These parent-child relationships are invariant
when running the algorithm. We also use a subroutine for
{\em duplicating the jointree nodes rooted at some node \(i\).}
This subroutine duplicates node \(i\) and
its descendant while also duplicating the edges connecting these nodes. 
If a duplicated node \(j\) hosts a family \(f\), 
this subroutine will make \(\dup{j}\)
host the duplicate family \(\dup{f}\) (so \(j\in\hosts(f)\)
iff \(\dup{j} \in \hosts(\dup{f})\)).

\begin{algorithm}[tb]
\begin{small}
\caption{Jointree to Twin Jointree}
\label{alg:jt}
\begin{algorithmic}[1]
\Procedure{Make-Twin-Jointree}{$\JT,r,p$}
\State \(\Sigma \gets\) leaf nodes at or below node \(r\) 
\If {nodes in \(\Sigma\) only host families for root variables}
\State \Return
\EndIf	
\If {nodes in \(\Sigma\) only host families for internal variables}
\State duplicate the jointree nodes rooted at node $r$
\State add $\dup{r}$ as a child of $p$ 
\Else
\For{each child $k$ of node $r$}
\State $\Call{Make-Twin-Jointree}{\JT,k,r}$
\EndFor
\EndIf	
\EndProcedure
\end{algorithmic}
\end{small}
\end{algorithm}

The conversion process is given in Algorithm~\ref{alg:jt} which
should be called initially with a root \(r\) that does not host a 
family for an internal DAG node and \(p=null\).
The twin jointree in Figure~\ref{fig:twin-jt} was obtained
from the base jointree in Figure~\ref{fig:base-jt} by this algorithm
which simply adds nodes and edges to the base jointree. If an edge \((i,j)\)
in the base jointree is duplicated by Algorithm~\ref{alg:jt}, 
we call \((i,j)\) a {\em duplicated edge} and \((\dup{i},\dup{j})\) 
a {\em duplicate edge.} Otherwise, we call \((i,j)\) an {\em invariant edge.} 
In Figure~\ref{fig:twin-jt}, duplicate edges are shown in red and 
invariant edges are shown in green. We now have the following key result
on these twin jointrees.

\begin{theorem}~\label{thm:twin-sep}
If the input jointree to Alg.~\ref{alg:jt} has separators
$\S$ and the output jointree has separators $\S^t$, then
for duplicated edges $(i,j)$, $\S^t_{ij} = \S_{ij}$;
for duplicate edges $(\dup{i},\dup{j})$, $\S^t_{\dup{i} \dup{j}} = \dup{\S_{ij}}$; and
for invariant edges $(i,j)$, $\S^t_{ij} = \S_{ij} \cup \dup{\S_{ij}}$.
\end{theorem}
One can verify that the separators in Figure~\ref{fig:alg-ex} satisfy 
these properties. The following result bounds the 
width and size of twin jointrees generated by Algorithm~\ref{alg:jt}.

\begin{corollary}~\label{cor:tjt-width}
Let \(w\) be the width of a jointree for base network $G$ and let
\(n\) be the number of jointree nodes.
Calling Algorithm~\ref{alg:jt} on this jointree will generate
a jointree for twin network \(G^t\) whose width is no 
greater than $2w+1$ and whose number of nodes is no greater than $2n$.
\end{corollary}
The above bound on width is tight as shown in Appendix~\ref{app:tight}.
Since treewidth can be defined in terms of jointree width, the
above result leads to the same guarantee of
Corollary~\ref{cor:twin-treewidth} on the treewidth of
twin networks.
However, the main role of the construction in this 
section is in bounding the causal treewidth of
twin networks. This is discussed next.

\section{The Causal Treewidth of Twin Networks}
\label{sec:ctw}

Recall that causal treewidth is a more refined notion than treewidth
as it uses more information about the network. In particular, this notion is relevant when we know that some variables in the
network are functional, without needing
to know the specific functions (equations) of these variables. By exploiting this information, one can construct thinned jointrees
that have smaller separators and clusters compared to classical jointrees, which can lead to exponential savings in reasoning time~\cite{uai/ChenDarwiche22,causalityAC,DarwicheECAI20b}.
As mentioned earlier, the causal treewidth corresponds to the 
minimum width of any thinned jointree. This is guaranteed to be no greater
than treewidth and can be bounded when treewidth is not~\cite{causalityAC}. We next show that the causal treewidth of a twin network is also at most twice the causal treewidth of its base network plus one. 

\begin{theorem}
\label{thm:thin-sep}
Consider a twin jointree constructed by Algorithm~\ref{alg:jt}
from a base jointree with thinned separators \(\S\). 
The following are valid thinned separators for this twin jointree:
for duplicated edges $(i,j)$, $\S^{t}_{ij} = \S_{ij}$; 
for duplicate edges $(\dup{i},\dup{j})$; $\S^{t}_{\dup{i}\dup{j}} = \dup{\S_{ij}}$; and
for invariant edges $(i,j)$, $\S^{t}_{ij} = \S_{ij} \cup \dup{\S_{ij}}$.
\end{theorem}
This theorem shows that a thinned, base jointree can be easily 
converted into a thinned, twin jointree.
This is significant for two reasons. First,
this method avoids the explicit construction of thinned jointrees 
for twin networks which can be quite expensive computationally~\cite{uai/ChenDarwiche22}.
Second, we have the following guarantee on the width of thinned, twin
jointrees constructed by Theorem~\ref{thm:thin-sep}.

\begin{corollary}\label{cor:thinning-1}
Consider the thinned, base and twin jointrees in Theorem~\ref{thm:thin-sep}.
If the thinned, base jointree has width \(w\), then the thinned, twin jointree
has width no greater than \(2w+1\).
\end{corollary}

Due to space constraints,
we include a thinned jointree for the base network
and the corresponding thinned, twin jointree constructed by
Algorithm~\ref{alg:jt} and Theorem~\ref{thm:thin-sep} in 
Appendix~\ref{app:fig}.
We can now bound the causal treewidth of twin networks.

\begin{corollary}~\label{cor:causal-tw}
If \(w\) and \(w^t\) are the causal treewidths of a base network
and its twin network, then $w^t \leq 2w+1$.
\end{corollary}

\section{Counterfactual Reasoning Beyond Two Worlds}
\label{sec:n-worlds}

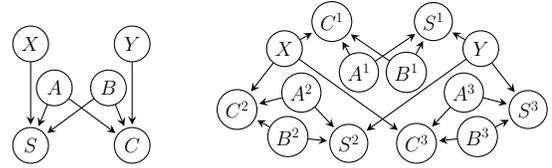
\begin{figure}
\centering
\centering
\begin{tikzpicture}[->,>=stealth,shorten >=1pt,auto,scale=0.45,thin,font=\huge,transform shape]
\node[state,thin,font=\huge] (X) at (-1.5,0) {$X$};
\node[state,thin,font=\huge] (Y) at (1.5,0) {$Y$};
\node[state,thin,font=\huge] (A) at (-0.8,-1.3) {$A$};
\node[state,thin,font=\huge] (B) at (0.8,-1.3) {$B$};
\node[state,thin,font=\huge] (S) at (-1.5,-3) {$S$};
\node[state,thin,font=\huge] (C) at (1.5,-3) {$C$};
\path (A) edge (S);
\path (B) edge (S);
\path (X) edge (S);
\path (A) edge (C);
\path (B) edge (C);
\path (Y) edge (C);
\end{tikzpicture}
\quad\quad
\begin{tikzpicture}[->,>=stealth,shorten >=1pt,auto,scale=0.45,transform shape]
\node[state,thin,font=\huge] (Y) at (-2.9,0) {$X$};
\node[state,thin,font=\huge] (X) at (2.9,0) {$Y$};

\node[state,thin,font=\huge] (A1) at (-0.7,-0.7) {$\ndup{A}{1}$};
\node[state,thin,font=\huge] (B1) at (0.7,-0.7) {$\ndup{B}{1}$};
\node[state,thin,font=\huge] (S1) at (1.5,0.8) {$\ndup{S}{1}$};
\node[state,thin,font=\huge] (C1) at (-1.5,0.8) {$\ndup{C}{1}$};

\node[state,thin,font=\huge] (A2) at (-2.4,-1.3) {$\ndup{A}{2}$};
\node[state,thin,font=\huge] (B2) at (-2.8,-2.6) {$\ndup{B}{2}$};
\node[state,thin,font=\huge] (S2) at (-1,-2.8) {$\ndup{S}{2}$};
\node[state,thin,font=\huge] (C2) at (-4.3,-1.8) {$\ndup{C}{2}$};

\node[state,thin,font=\huge] (A3) at (2.4,-1.3) {$\ndup{A}{3}$};
\node[state,thin,font=\huge] (B3) at (2.8,-2.6) {$\ndup{B}{3}$};
\node[state,thin,font=\huge] (S3) at (4.3,-1.8) {$\ndup{S}{3}$};
\node[state,thin,font=\huge] (C3) at (1,-2.8) {$\ndup{C}{3}$};

\path (A1) edge (S1);
\path (B1) edge (S1);
\path (X) edge (S1);
\path (A1) edge (C1);
\path (B1) edge (C1);
\path (Y) edge (C1);

\path (A2) edge (S2);
\path (B2) edge (S2);
\path (X) edge (S2);
\path (A2) edge (C2);
\path (B2) edge (C2);
\path (Y) edge (C2);

\path (A3) edge (S3);
\path (B3) edge (S3);
\path (X) edge (S3);
\path (A3) edge (C3);
\path (B3) edge (C3);
\path (Y) edge (C3);

\end{tikzpicture}
\caption{A base network and its 3-world network.
\label{fig:half-adder-bn2}
\label{fig:half-adder-3tn}}
\end{figure}

Standard counterfactual reasoning contemplates two
worlds, one real and another imaginary, while assuming
that exogenous variables
correspond to causal mechanisms that govern both worlds. 
This motivates the notion of a twin network as it ensures 
that these causal mechanisms are invariant. 
We can think
of counterfactual reasoning as a kind of {\em temporal reasoning}
where endogenous variables can change their states over time. 
A more general setup arises when we allow some exogenous variables to change their states over time.
For example, consider again the half
adder in Figure~\ref{fig:half-adder} and its base network
in Figure~\ref{fig:half-adder-bn2}. 
Suppose we set inputs $A$ and $B$ to high and low
and observe outputs $S$ and $C$ to be high and low,
which is a normal behavior.
We then set both inputs to low and observe
that the outputs do not change, which is an abnormal behavior.
We then aim to predict the state of outputs if we were to set both
inputs to high. This scenario involves three time
steps (worlds). Moreover, while the health
of gates \(X\) and \(Y\) are invariant over time, we
do not wish to make the same assumption about the inputs
\(A\) and \(B\). 
We can model this situation using the
network in Figure~\ref{fig:half-adder-3tn}, which is
a more general
type of networks that we call 
\(N\)-world networks.

\begin{definition}\label{def:n-twin}
Consider a base network $G$ and let $\R$ be a subset of its
roots and \(N \geq 1\) be an integer. The \underline{$N$-world network} $G^N$ of \(G\) 
is constructed as follows.
For each variable $X$ in $G$ that is not in $\R$, replace it with N duplicates of $X$, labeled $\ndup{X}{1},\ndup{X}{2},\dots,\ndup{X}{N}$. For each parent $P$ of $X$, if $P$ is in $\R$, make $P$ a parent of $\ndup{X}{i}$ for all $i \in 1,2,\dots,N$. Otherwise, make $\ndup{P}{i}$ a parent of $\ndup{X}{i}$ for all $i \in 1,2,\dots,N$.
\end{definition}

This definition corresponds to the notion of a
{\em parallel worlds model}~\cite{ijcai/AvinSP05} 
when \(\R\) contains all roots
in the base network. Moreover, twin networks fall as a special 
case when \(N=2\) and \(\R\) contains all 
roots of the base network. 
We next bound the (causal) treewidth of \(N\)-world networks
by the (causal) treewidth of their base networks.

\begin{theorem}~\label{thm:ntwin-tw}
If \(w\) and \(w^t\) are the (causal) treewidths of a base network and 
its \(N\)-world network, then \(w^t \leq N(w+1)-1\).
\end{theorem}

The class of \(N\)-world networks
is a subclass
of {\em dynamic Bayesian networks}~\cite{Dean} 
and is significant for a number of reasons. 
First, as illustrated above,
it arises when reasoning about the behavior of systems
consisting of function blocks (e.g., gates)~\cite{MBD-book}.
These kinds of physical systems can be easily modeled
using fully specified SCMs, where the structural equations
correspond to component behaviors and the distributions over exogenous variables
correspond to component reliabilities; see~\cite[Ch~5]{DarwicheBook09}
for a textbook discussion and~\cite{tsmc/MengshoelCCPDU10} for a case study of
a real-world electrical power system. More broadly, \(N\)-world networks allow counterfactual
reasoning that involves conflicting observations and actions that 
arise in multiple worlds as in the {\em unit selection problem}~\cite{corr/LiP22a}---for example,~\cite{Huang-Unit} used Theorem~\ref{thm:ntwin-tw} to obtain bounds on the complexity of this problem.
See also~\cite{ijcai/AvinSP05,uai/ShpitserP07,jmlr/ShpitserP08} 
for further applications of \(N\)-world 
networks in the context of counterfactual reasoning. 
Appendix~\ref{app:gn-world-net}
shows that the treewidth bound of
Theorem~\ref{thm:ntwin-tw} holds for a generalization of $N$-world networks that permits the duplication of only a subset of base nodes and allows certain edges that extend between worlds.

Our complexity bounds thus far apply to any counterfactual query. 
For a specific counterfactual query,
we can further reduce the complexity of inference
by pruning nodes and edges as in~\cite[Ch~6]{DarwicheBook09}
and merging nodes which leads to {\em counterfactual graphs} as in~\cite{uai/ShpitserP07}. 

\section{Counterfactual Reasoning with Partially Specified SCMs}
\label{sec:pscms}

The results we presented on \(N\)-world networks, which
include twin networks, apply directly to fully specified SCMs. In particular,
in the context of variable elimination and jointree algorithms, these results
allow us to bound the complexity of computing counterfactual queries  
in terms of the complexity of computing associational/interventional
queries. Moreover, they provide efficient methods for constructing 
elimination orders and jointrees that can be used for computing counterfactual queries 
based on the ones used for answering associational/interventional queries, 
while ensuring that the stated bounds will be realized. Recall again that our
bounds and constructions apply to both traditional treewidth 
and the more 
recent causal treewidth.

Causal reasoning can also be conducted on
partially specified SCMs and data, which is a more common and challenging
task. 
A partially specified SCM typically includes the
SCM structure and some information about its 
parameters (i.e., its structural equations and the distributions 
over its exogenous variables). 
For example, we may not know any of the SCM paramaters, 
or we may know the structural equations but not the
distributions over exogenous variables as assumed in~\cite{nips/Zaffalon}.
A central question in this setup is whether the available
information, which includes data, is sufficient to obtain a point estimate for the causal
query of interest, in which case the query is said to be identifiable. A significant
amount of work has focused on characterizing conditions under
which causal queries (both counterfactual and interventional) are identifiable; see,
\cite{pearl00b,SpirtesBook} for textbook discussions of this subject and
\cite{jmlr/ShpitserP08,nips/CorreaLB21} for some results on the 
identification of counterfactual queries.

When a query is identifiable, the classical approach for estimating it
is to derive an estimand using techniques such as the do-calculus
for interventional queries~\cite{do-calculus-95,aaai/TianP02,aaai/ShpitserP06}.\footnote{See \cite{aaai/Jung0B21,icml/Jung0B21,nips/Jung0B20} for
some recent work on estimating identifiable interventional queries from finite data.} 
Some recent approaches take a different direction by first estimating 
the SCM parameters to yield a fully specified SCM that is then used to answer (identifiable)
interventional and counterfactual queries using classical inference 
algorithms~\cite{bias/Zaffalon,nips/Zaffalon,causalityAC}.
Our results on twin and \(N\)-world networks apply directly
in this case as they can be used when conducting inference on the fully 
parameterized SCM.
For unidentifiable queries, the classical approach is
to derive a closed-form bound on the query; see, for example,~\cite{Balke1994,synthese/Pearl99,amai/TianP00,dawid2017,Rosset2018,Evans2018,Zhang2021,ijcai/MuellerLP22}. 
Some recent approaches take a different direction for establishing
bounds, such as reducing the problem into one of polynomial programming~\cite{corr/discrete,icml/Zhang0B22}
or inference on credal networks~\cite{Zaffalon0C20,ai/Cozman00,MAUA2020133}.
Another recent direction is to establish (approximate) bounds by
estimating SCM parameters and then using classical inference 
algorithms on the fully specified 
SCM to obtain point estimates~\cite{nips/Zaffalon,bias/Zaffalon}. 
Since the query is not identifiable, different parametrizations can lead
to different point estimates which are employed to improve (widen) 
the computed bounds. 
Our results can also be used in this case for computing
point estimates based on a particular 
parametrization (fully specified SCM) within the overall process of
establishing bounds. 

\section{Experimental Results}
\label{sec:expS}
We consider experiments that target random
networks whose structures emulate the structures of SCMs used in 
counterfactual reasoning. We have a few objectives in mind. 
First, we wish to compare the widths of base and twin jointrees,
with and without thinning. These widths do not correspond
to (causal) treewidth since the jointrees are constructed using
heuristics (finding optimal jointrees is NP-hard).
Next, we want to compare the quality of twin jointrees constructed
by Algorithm~\ref{alg:jt} (TWIN-ALG1), which operates directly on a base jointree, to the quality of twin jointrees obtained by applying the minfill heuristic to a twin network (TWIN-MF). Recall that the former
method is more efficient than the latter method.
Finally, we wish to conduct a similar comparison between the thinned, twin 
jointrees constructed according to Theorem~\ref{thm:thin-sep} (TWIN-THM3) 
and the thinned, twin jointrees obtained by applying the minfill
heuristic and thinning rules to a twin network (TWIN-MF-RLS).
Again, the former method is more efficient than the latter.
The widths of these jointrees will be compared to the widths
of base jointrees constructed by minfill (BASE-MF)
and thinned, base jointrees constructed by minfill and
thinning rules (BASE-MF-RLS).

We generated random networks according to the method used
in~\cite{DarwicheECAI20b}. Given a number of nodes \(n\)
and a maximum number of parents \(p\), the method chooses
the parents of node \(X_i\) randomly from the set \(X_1,\ldots,X_{i-1}\).
The number of parents for node \(X_i\) is chosen randomly from the set
\(0,\ldots,min(p,i-1)\). We refer to these networks as rNET.
We then consider each internal node \(N\) and add a unique 
root \(R\) as parent for \(N\). This is meant
to emulate the structure of SCMs as the exogenous variable \(R\)
can be viewed as representing the different causal mechanisms for
endogenous variable \(N\). We refer to these modified networks
as rSCM. The twin networks of rSCM are more complex than 
those for rNET since more variables are shared between the two
slices representing the real and imaginary worlds (i.e., more 
information is shared between the two worlds).
We used 
$n \in \{50,75,100,125,150,200,250,300\}$ and 
$p \in \{3,5,7\}$. For each combination of \(n\) and \(p\),
we generated \(50\) random, base networks and reported averages
of two properties for the constructed jointrees: 
width and {\em normalized width.} If a jointree
has clusters \(\C_1, \ldots, \C_n\), then normalized width is
\(\log_2 \sum_{i=1}^n 2^{|\C_i|}\).
This accounts for all clusters in the jointree (instead of just 
the largest one)
and the jointree size.
The data we generated occupies significant
space so we included it in Appendix~\ref{app:exp} while providing
representative plots in Figure~\ref{fig:exp1} for jointree widths 
under \(p=5\).
We next discuss patterns exhibited in these plots and 
the full data in Appendix~\ref{app:exp}, which also includes experiments
using random networks generated according to the method in~\cite{sbia/IdeC02}. 

First, the widths of twin jointrees are always less than twice
the widths of their base jointrees and often significantly 
less than that. This is not guaranteed by our theoretical
bounds as those apply to (causal) treewidth not to the widths
of jointrees produced by heuristics --- the latter widths are an
upper bound on the former. Second, constructing a twin jointree by directly applying Algorithm~\ref{alg:jt} to a base jointree (TWIN-ALG1) 
is relatively comparable to constructing the 
twin jointree by operating on the twin network (TWIN-MF),
as would normally be done. 
This also holds for thinned jointrees (TWIN-THM3 vs TWIN-MF-RLS)
and is encouraging since the former methods are much more 
efficient than the latter ones.  Third, the employment of thinned jointrees can lead to significant
reduction in width and hence an exponential reduction in reasoning
time. This can be seen by comparing the widths of 
twin jointrees TWIN-THM3 and TWIN-ALG1 since the former is thinned
but the latter is not (similarly for TWIN-MF-RLS and TWIN-MF). Fourth, the twin jointrees of rSCM have larger widths than
those of rNET. Recall that in rSCM, every endogenous variable
has its own exogenous variable as a parent.
Therefore, the distribution over exogenous variables
has a larger space in rSCM compared to rNET.
Since this distribution needs to be shared between the real
and imaginary worlds, counterfactual reasoning with rSCM 
is indeed expected to be more complex computationally than
reasoning with rNET. Finally, consider Figure~\ref{fig:rand-scm-thinned-jt}
for a bottom-line comparison between the complexity 
of counterfactual reasoning and the complexity of
associational/interventional reasoning in practice. 
Jointrees BASE-MF have the smallest widths for base networks
so these are the jointrees one would use for 
associational/interventional reasoning.
The best twin jointrees are TWIN-MF-RLS which are thinned.
This is what one would use for counterfactual reasoning.
The widths of latter jointrees are always less than twice
the widths of the former, and quite often significantly 
much less.\footnote{See footnote~\ref{foot:replication}
for why BASE-MF is better than BASE-MF-RLS for rSCM.}
\begin{figure}[tb]
\begin{subfigure}[b]{0.47\columnwidth}
\centering
\includegraphics[width=.9\linewidth]{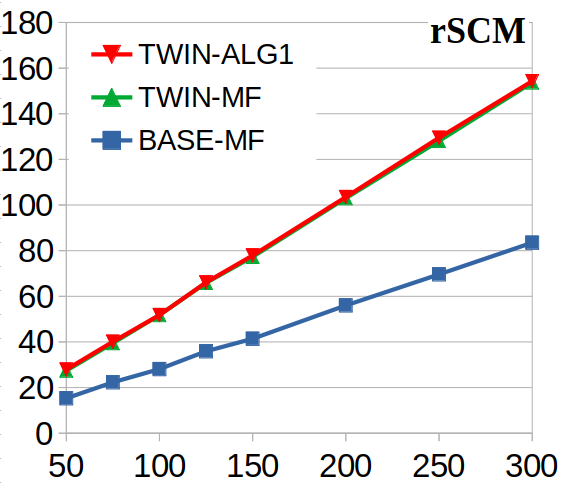}
\caption{classical jointrees
\label{fig:rand-scm-jt}}
\end{subfigure}
\hfill
\begin{subfigure}[b]{0.47\columnwidth}
\centering
\includegraphics[width=.9\linewidth]{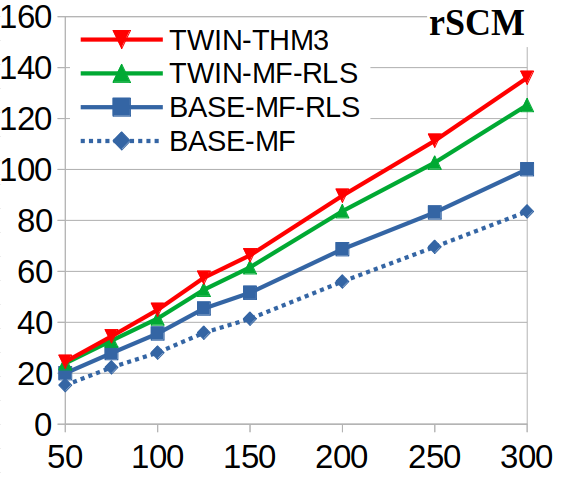}
\caption{thinned jointrees
\label{fig:rand-scm-thinned-jt}}
\end{subfigure}
\hfill
\caption{Width of jointrees (y-axis) against number of base 
network nodes (x-axis) for maximum number of parents \(p=5\).
\label{fig:exp1}
\label{fig:exp2}
}
\end{figure}
\section{Conclusion}
We studied the complexity of counterfactual reasoning
on fully specified SCMs in relation to the complexity of
associational and interventional reasoning on these models. 
Our basic finding is that in the context of algorithms
based on (causal) treewidth, the former complexity 
is no greater than quadratic in the latter when counterfactual
reasoning involves only two worlds.
We extended these results to counterfactual reasoning that requires
multiple worlds, showing that the gap in complexity is bounded
polynomially by the number of needed worlds.
Our empirical results suggest that for two types of random SCMs,
the complexity of counterfactual reasoning is closer
to that of associational and interventional reasoning 
than our worst-case theoretical analysis may suggest.
While our results directly target counterfactual reasoning
on fully specified SCMs, we also discussed cases when they
can be applied to counterfactual reasoning on partially 
specified SCMs that are coupled with data.

\appendix

\section*{Acknowledgements}
We wish to thank Haiying Huang, Scott Mueller, Judea Pearl, Ilya Shpitser, Jin Tian and Marco
Zaffalon for providing useful feedback on an earlier version of this paper. This work has been partially
supported by ONR grant N000142212501.
\bibliographystyle{named}
\bibliography{ijcai23}

\begin{thebibliography}{}

\bibitem[\protect\citeauthoryear{Avin \bgroup \em et al.\egroup
  }{2005}]{ijcai/AvinSP05}
Chen Avin, Ilya Shpitser, and Judea Pearl.
\newblock Identifiability of path-specific effects.
\newblock In {\em {IJCAI}}, pages 357--363. Professional Book Center, 2005.

\bibitem[\protect\citeauthoryear{Balke and Pearl}{1994a}]{Balke1994}
Alexander Balke and Judea Pearl.
\newblock Counterfactual probabilities: Computational methods, bounds and
  applications.
\newblock In {\em {UAI}}, pages 46--54. Morgan Kaufmann, 1994.

\bibitem[\protect\citeauthoryear{Balke and Pearl}{1994b}]{aaai/BalkeP94}
Alexander Balke and Judea Pearl.
\newblock Probabilistic evaluation of counterfactual queries.
\newblock In {\em {AAAI}}, pages 230--237. {AAAI} Press / The {MIT} Press,
  1994.

\bibitem[\protect\citeauthoryear{Balke and Pearl}{1995}]{uai/BalkeP95}
Alexander Balke and Judea Pearl.
\newblock Counterfactuals and policy analysis in structural models.
\newblock In {\em {UAI}}, pages 11--18. Morgan Kaufmann, 1995.

\bibitem[\protect\citeauthoryear{Bareinboim \bgroup \em et al.\egroup
  }{2021}]{Bareinboim20211OP}
E.~Bareinboim, Juan~David Correa, D.~Ibeling, and Thomas~F. Icard.
\newblock On {Pearl}'s hierarchy and the foundations of causal inference.
\newblock 2021.
\newblock Technical Report, R-60, Colombia University.

\bibitem[\protect\citeauthoryear{Chen and Darwiche}{2022}]{uai/ChenDarwiche22}
Yizuo Chen and Adnan Darwiche.
\newblock On the definition and computation of causal treewidth.
\newblock In {\em {UAI}}, volume 180 of {\em Proceedings of Machine Learning
  Research}, pages 368--377. {PMLR}, 2022.

\bibitem[\protect\citeauthoryear{Correa \bgroup \em et al.\egroup
  }{2021}]{nips/CorreaLB21}
Juan~D. Correa, Sanghack Lee, and Elias Bareinboim.
\newblock Nested counterfactual identification from arbitrary surrogate
  experiments.
\newblock In {\em NeurIPS}, pages 6856--6867, 2021.

\bibitem[\protect\citeauthoryear{Cozman}{2000}]{ai/Cozman00}
F{\'{a}}bio~Gagliardi Cozman.
\newblock Credal networks.
\newblock {\em Artif. Intell.}, 120(2):199--233, 2000.

\bibitem[\protect\citeauthoryear{Darwiche}{2003}]{DarwicheJACM03}
Adnan Darwiche.
\newblock A differential approach to inference in {B}ayesian networks.
\newblock {\em J. {ACM}}, 50(3):280--305, 2003.

\bibitem[\protect\citeauthoryear{Darwiche}{2009}]{DarwicheBook09}
Adnan Darwiche.
\newblock {\em Modeling and Reasoning with {B}ayesian Networks}.
\newblock Cambridge University Press, 2009.

\bibitem[\protect\citeauthoryear{Darwiche}{2020}]{DarwicheECAI20b}
Adnan Darwiche.
\newblock An advance on variable elimination with applications to tensor-based
  computation.
\newblock In {\em {ECAI}}, volume 325 of {\em Frontiers in Artificial
  Intelligence and Applications}, pages 2559--2568. {IOS} Press, 2020.

\bibitem[\protect\citeauthoryear{Darwiche}{2021}]{causalityAC}
Adnan Darwiche.
\newblock Causal inference with tractable circuits.
\newblock In {\em WHY Workshop, NeurIPS}, 2021.
\newblock https://arxiv.org/abs/2202.02891.

\bibitem[\protect\citeauthoryear{Darwiche}{2022}]{neusys22}
Adnan Darwiche.
\newblock Tractable {B}oolean and arithmetic circuits.
\newblock In Pascal Hitzler and Md~Kamruzzaman Sarker, editors, {\em
  Neuro-symbolic Artificial Intelligence: The State of the Art}, volume 342,
  chapter~6. Frontiers in Artificial Intelligence and Applications. IOS Press,
  2022.

\bibitem[\protect\citeauthoryear{Dawid \bgroup \em et al.\egroup
  }{2017}]{dawid2017}
Philip Dawid, Monica Musio, and Rossella Murtas.
\newblock The probability of causation.
\newblock {\em Law, Probability and Risk}, (16):163--179, 2017.

\bibitem[\protect\citeauthoryear{Dean and Kanazawa}{1989}]{Dean}
Thomas Dean and Keiji Kanazawa.
\newblock A model for reasoning about persistence and causation.
\newblock {\em Computational Intelligence}, 5, 1989.

\bibitem[\protect\citeauthoryear{Dechter}{1996}]{dechterUAI96}
Rina Dechter.
\newblock Bucket elimination: A unifying framework for probabilistic inference.
\newblock In {\em Proceedings of the Twelfth Annual Conference on Uncertainty
  in Artificial Intelligence (UAI)}, pages 211--219, 1996.

\bibitem[\protect\citeauthoryear{Duarte \bgroup \em et al.\egroup
  }{2021}]{corr/discrete}
Guilherme Duarte, Noam Finkelstein, Dean Knox, Jonathan Mummolo, and Ilya
  Shpitser.
\newblock An automated approach to causal inference in discrete settings.
\newblock {\em CoRR}, abs/2109.13471, 2021.

\bibitem[\protect\citeauthoryear{Evans}{2018}]{Evans2018}
Robin~J. Evans.
\newblock {Margins of discrete Bayesian networks}.
\newblock {\em The Annals of Statistics}, 46(6A):2623 -- 2656, 2018.

\bibitem[\protect\citeauthoryear{Galles and Pearl}{1998}]{galles1998axiomatic}
David Galles and Judea Pearl.
\newblock An axiomatic characterization of causal counterfactuals.
\newblock {\em Foundations of Science}, 3(1):151--182, 1998.

\bibitem[\protect\citeauthoryear{Halpern}{2000}]{halpern2000axiomatizing}
Joseph~Y Halpern.
\newblock Axiomatizing causal reasoning.
\newblock {\em Journal of Artificial Intelligence Research}, 12:317--337, 2000.

\bibitem[\protect\citeauthoryear{Hamscher \bgroup \em et al.\egroup
  }{1992}]{MBD-book}
Walter Hamscher, Luca Console, and Johan de~Kleer, editors.
\newblock {\em Readings in Model-Based Diagnosis}.
\newblock Morgan Kaufmann Publishers Inc., San Francisco, CA, USA, 1992.

\bibitem[\protect\citeauthoryear{Huang and Darwiche}{2023}]{Huang-Unit}
Haiying Huang and Adnan Darwiche.
\newblock An algorithm and complexity results for causal unit selection.
\newblock In {\em 2nd Conference on Causal Learning and Reasoning}, Proceedings
  of Machine Learning Research, 2023.

\bibitem[\protect\citeauthoryear{Ide and Cozman}{2002}]{sbia/IdeC02}
Jaime~Shinsuke Ide and F{\'{a}}bio~Gagliardi Cozman.
\newblock Random generation of bayesian networks.
\newblock In {\em {SBIA}}, volume 2507 of {\em Lecture Notes in Computer
  Science}, pages 366--375. Springer, 2002.

\bibitem[\protect\citeauthoryear{Jung \bgroup \em et al.\egroup
  }{2020}]{nips/Jung0B20}
Yonghan Jung, Jin Tian, and Elias Bareinboim.
\newblock Learning causal effects via weighted empirical risk minimization.
\newblock In {\em NeurIPS}, 2020.

\bibitem[\protect\citeauthoryear{Jung \bgroup \em et al.\egroup
  }{2021a}]{icml/Jung0B21}
Yonghan Jung, Jin Tian, and Elias Bareinboim.
\newblock Estimating identifiable causal effects on markov equivalence class
  through double machine learning.
\newblock In {\em {ICML}}, volume 139 of {\em Proceedings of Machine Learning
  Research}, pages 5168--5179. {PMLR}, 2021.

\bibitem[\protect\citeauthoryear{Jung \bgroup \em et al.\egroup
  }{2021b}]{aaai/Jung0B21}
Yonghan Jung, Jin Tian, and Elias Bareinboim.
\newblock Estimating identifiable causal effects through double machine
  learning.
\newblock In {\em {AAAI}}, pages 12113--12122. {AAAI} Press, 2021.

\bibitem[\protect\citeauthoryear{Kjaerulff}{1990}]{ELMheuristics}
Uffe Kjaerulff.
\newblock Triangulation of graphs – algorithms giving small total state
  space.
\newblock 1990.
\newblock Technical Report R-90-09, Department of Mathematics and Computer
  Science, University of Aalborg, Denmark.

\bibitem[\protect\citeauthoryear{Kj{\ae}rulff}{1994}]{Kjaerulff1994}
Uffe Kj{\ae}rulff.
\newblock Reduction of computational complexity in bayesian networks through
  removal of weak dependences.
\newblock In {\em {UAI}}, pages 374--382. Morgan Kaufmann, 1994.

\bibitem[\protect\citeauthoryear{Koller and Friedman}{2009}]{kollerbook}
Daphne Koller and Nir Friedman.
\newblock {\em Probabilistic Graphical Models - Principles and Techniques}.
\newblock {MIT} Press, 2009.

\bibitem[\protect\citeauthoryear{Larrañaga \bgroup \em et al.\egroup
  }{1997}]{Larranaga1997}
Pedro Larrañaga, Cindy M.~H. Kuijpers, Mikel Poza, and Roberto~H. Murga.
\newblock Decomposing bayesian networks: triangulation of the moral graph with
  genetic algorithms.
\newblock {\em Statistics and Computing}, 7, 1997.

\bibitem[\protect\citeauthoryear{Lauritzen and
  Spiegelhalter}{1988}]{jointreeAlg}
S.~L. Lauritzen and D.~J. Spiegelhalter.
\newblock Local computations with probabilities on graphical structures and
  their application to expert systems.
\newblock {\em Journal of the Royal Statistical Society. Series B
  (Methodological)}, 50(2):157--224, 1988.

\bibitem[\protect\citeauthoryear{Li and Pearl}{2022}]{corr/LiP22a}
Ang Li and Judea Pearl.
\newblock Unit selection with nonbinary treatment and effect.
\newblock {\em CoRR}, abs/2208.09569, 2022.

\bibitem[\protect\citeauthoryear{Mauá and Cozman}{2020}]{MAUA2020133}
Denis~Deratani Mauá and Fabio~Gagliardi Cozman.
\newblock Thirty years of credal networks: Specification, algorithms and
  complexity.
\newblock {\em International Journal of Approximate Reasoning}, 126:133--157,
  2020.

\bibitem[\protect\citeauthoryear{Mengshoel \bgroup \em et al.\egroup
  }{2010}]{tsmc/MengshoelCCPDU10}
Ole~J. Mengshoel, Mark Chavira, Keith Cascio, Scott Poll, Adnan Darwiche, and
  N.~Serdar Uckun.
\newblock Probabilistic model-based diagnosis: An electrical power system case
  study.
\newblock {\em {IEEE} Trans. Syst. Man Cybern. Part {A}}, 40(5):874--885, 2010.

\bibitem[\protect\citeauthoryear{Mueller \bgroup \em et al.\egroup
  }{2022}]{ijcai/MuellerLP22}
Scott Mueller, Ang Li, and Judea Pearl.
\newblock Causes of effects: Learning individual responses from population
  data.
\newblock In {\em {IJCAI}}, pages 2712--2718. ijcai.org, 2022.

\bibitem[\protect\citeauthoryear{Pearl and Mackenzie}{2018}]{pearl18}
Judea Pearl and Dana Mackenzie.
\newblock {\em The Book of Why: The New Science of Cause and Effect}.
\newblock Basic Books, 2018.

\bibitem[\protect\citeauthoryear{Pearl}{1988}]{Pearl88b}
Judea Pearl.
\newblock {\em Probabilistic Reasoning in Intelligent Systems: Networks of
  Plausible Inference}.
\newblock Morgan Kaufmann, 1988.

\bibitem[\protect\citeauthoryear{Pearl}{1995}]{do-calculus-95}
Judea Pearl.
\newblock Causal diagrams for empirical research.
\newblock {\em Biometrika}, 82(4):669--688, 1995.

\bibitem[\protect\citeauthoryear{Pearl}{1999}]{synthese/Pearl99}
Judea Pearl.
\newblock Probabilities of causation: Three counterfactual interpretations and
  their identification.
\newblock {\em Synth.}, 121(1-2):93--149, 1999.

\bibitem[\protect\citeauthoryear{Pearl}{2009}]{pearl00b}
Judea Pearl.
\newblock {\em Causality}.
\newblock Cambridge University Press, second edition, 2009.

\bibitem[\protect\citeauthoryear{Peters \bgroup \em et al.\egroup
  }{2017}]{PetersBook}
Jonas Peters, Dominik Janzing, and Bernhard Schölkopf.
\newblock {\em Elements of Causal Inference: Foundations and Learning
  Algorithms}.
\newblock MIT Press, 2017.

\bibitem[\protect\citeauthoryear{Robertson and
  Seymour}{1986}]{jal/RobertsonS86}
Neil Robertson and Paul~D. Seymour.
\newblock Graph minors. {II.} algorithmic aspects of tree-width.
\newblock {\em J. Algorithms}, 7(3):309--322, 1986.

\bibitem[\protect\citeauthoryear{Rosset \bgroup \em et al.\egroup
  }{2018}]{Rosset2018}
Denis Rosset, Nicolas Gisin, and Elie Wolfe.
\newblock Universal bound on the cardinality of local hidden variables in
  networks.
\newblock {\em Quantum Info. Comput.}, 18(11–12):910–926, sep 2018.

\bibitem[\protect\citeauthoryear{Shpitser and Pearl}{2006}]{aaai/ShpitserP06}
Ilya Shpitser and Judea Pearl.
\newblock Identification of joint interventional distributions in recursive
  semi-markovian causal models.
\newblock In {\em {AAAI}}, pages 1219--1226. {AAAI} Press, 2006.

\bibitem[\protect\citeauthoryear{Shpitser and Pearl}{2007}]{uai/ShpitserP07}
Ilya Shpitser and Judea Pearl.
\newblock What counterfactuals can be tested.
\newblock In {\em {UAI}}, pages 352--359. {AUAI} Press, 2007.

\bibitem[\protect\citeauthoryear{Shpitser and Pearl}{2008}]{jmlr/ShpitserP08}
Ilya Shpitser and Judea Pearl.
\newblock Complete identification methods for the causal hierarchy.
\newblock {\em J. Mach. Learn. Res.}, 9:1941--1979, 2008.

\bibitem[\protect\citeauthoryear{Spirtes \bgroup \em et al.\egroup
  }{2000}]{SpirtesBook}
Peter Spirtes, Clark Glymour, and Richard Scheines.
\newblock {\em Causation, Prediction, and Search, Second Edition}.
\newblock Adaptive computation and machine learning. {MIT} Press, 2000.

\bibitem[\protect\citeauthoryear{Tian and Pearl}{2000}]{amai/TianP00}
Jin Tian and Judea Pearl.
\newblock Probabilities of causation: Bounds and identification.
\newblock {\em Ann. Math. Artif. Intell.}, 28(1-4):287--313, 2000.

\bibitem[\protect\citeauthoryear{Tian and Pearl}{2002}]{aaai/TianP02}
Jin Tian and Judea Pearl.
\newblock A general identification condition for causal effects.
\newblock In {\em {AAAI/IAAI}}, pages 567--573. {AAAI} Press / The {MIT} Press,
  2002.

\bibitem[\protect\citeauthoryear{Zaffalon \bgroup \em et al.\egroup
  }{2020}]{Zaffalon0C20}
Marco Zaffalon, Alessandro Antonucci, and Rafael Caba{\~{n}}as.
\newblock Structural causal models are (solvable by) credal networks.
\newblock In {\em International Conference on Probabilistic Graphical Models,
  {PGM}}, volume 138 of {\em Proceedings of Machine Learning Research}, pages
  581--592. {PMLR}, 2020.

\bibitem[\protect\citeauthoryear{Zaffalon \bgroup \em et al.\egroup
  }{2021}]{nips/Zaffalon}
Marco Zaffalon, Alessandro Antonucci, and Rafael Cabañas.
\newblock Causal {E}xpectation-{M}aximisation.
\newblock In {\em WHY Workshop, NeurIPS}, 2021.
\newblock https://arxiv.org/abs/2011.02912.

\bibitem[\protect\citeauthoryear{Zaffalon \bgroup \em et al.\egroup
  }{2022}]{bias/Zaffalon}
Marco Zaffalon, Alessandro Antonucci, Rafael Caba{\~{n}}as, David Huber, and
  Dario Azzimonti.
\newblock Bounding counterfactuals under selection bias.
\newblock In {\em {PGM}}, volume 186 of {\em Proceedings of Machine Learning
  Research}, pages 289--300. {PMLR}, 2022.

\bibitem[\protect\citeauthoryear{Zhang and Poole}{1994}]{uai/ZhangP94}
Nevin~Lianwen Zhang and David~L. Poole.
\newblock Intercausal independence and heterogeneous factorization.
\newblock In {\em {UAI}}, pages 606--614. Morgan Kaufmann, 1994.

\bibitem[\protect\citeauthoryear{Zhang \bgroup \em et al.\egroup
  }{2021}]{Zhang2021}
Junzhe Zhang, Jin Tian, and Elias Bareinboim.
\newblock Partial counterfactual identification from observational and
  experimental data.
\newblock {\em CoRR}, abs/2110.05690, 2021.

\bibitem[\protect\citeauthoryear{Zhang \bgroup \em et al.\egroup
  }{2022}]{icml/Zhang0B22}
Junzhe Zhang, Jin Tian, and Elias Bareinboim.
\newblock Partial counterfactual identification from observational and
  experimental data.
\newblock In {\em {ICML}}, volume 162 of {\em Proceedings of Machine Learning
  Research}, pages 26548--26558. {PMLR}, 2022.

\end{thebibliography}
\newpage

\onecolumn
\appendix
\section{Proofs} 
\label{app:proof}

\subsection*{Proof of Theorem 1}
Consider a base network $G$ and its twin network $G^t$. We first introduce a set notation $\{X,\dup{X}\}$, which contains both the base and duplicate variable for $X$ if $X$ is an internal variable and collapses to a single variable if $X$ is a root variable. 

For an elimination order $\pi$ on a base network $G$, \cite{DarwicheBook09} defines a \emph{graph sequence} $G_1,G_2,\dots,G_n$ induced by $\pi$, where $G_1$ is the moral graph for $G$, and $G_{i+1}$ is the result of eliminating $\pi(i)$ from $G_i$. Similarly, we define a \emph{twin graph sequence} $G^t_1,G^t_2,\dots,G^t_n$, where $G^t_1$ is the moral graph for $G^t$, and $G^t_{i+1}$ is the result of eliminating $\{\pi(i),\dup{\pi(i)}\}$ from $G^t_i$.

Consider $G_i$ in the graph sequence induced by $\pi$ on $G$. For each variable $X$ in $G_i$, let $G_i(X)$ be the set consisting of $X$ and its neighbors in $G_i$. Similarly, let $G^t_i(X)$ and $G^t_i(\dup{X})$ denote the set consisting of $X$ and its neighbors, and the set consisting of $\dup{X}$ and its neighbors in $G^t_i$, respectively. By definition, $G_i(\pi(i)) = \cls(\pi(i))$ and $G^t_i(\pi(i)) = \cls^t(\pi(i))$.

We first propose a Lemma that relates $G^t_i(X)$ and $G_i(X)$.

\begin{lemma}
\label{lem:neighbor-set}
Suppose we apply elimination order $\pi$ to a base network $G$ and apply $\pi^t$ to its twin network $G^t$. Then for each variable $X$ in $G_i$, we have $G^t_i(X) \subseteq G_i(X) \cup \dup{G_i(X)}$ and $G^t_i(\dup{X}) \subseteq G_i(X) \cup \dup{G_i(X)}$.
\end{lemma}
\begin{proof}
We will prove this by induction on $G^t_i$. The statement holds initially for $G^t_1$ by the definition of twin networks. Suppose the statement holds for $G^t_{i-1}$, i.e. $G^t_{i-1}(X) \subseteq G_{i-1}(X) \cup \dup{G_{i-1}(X)}$ and $G^t_{i-1}(\dup{X}) \subseteq G_{i-1}(X) \cup \dup{G_{i-1}(X)}$ for every $\{X,\dup{X}\}$ in $G^t_{i-1}$, we need to show the statement holds for $G^t_i$.

For simplicity, let $Y$ denote the variable being eliminated at step $i$, i.e. $Y=\pi(i-1)$. WLG, consider each base variable $X$ in $G^t_i$ (similar argument can be applied to each duplicate variable $\dup{X}$). $G^t_i(X)$ is affected by the elimination of $\{Y,\dup{Y}\}$ iff $X$ is a neighbor of $\{Y,\dup{Y}\}$ in $G^t_{i-1}$. Moreover, by induction, $X$ is a neighbor of $\{Y,\dup{Y}\}$ in $G^t_{i-1}$ only if $X$ is a neighbor of $Y$ in $G_{i-1}$.

When $X$ is not a neighbor of $Y$, then $G^t_i(X) = G^t_{i-1}(X)$ and $G_i(X) = G_{i-1}(X)$, so the statement holds.

When $X$ is a neighbor of $Y$, $G_i(X) = G_{i-1}(X) \cup G_{i-1}(Y) \setminus \{Y\}$ by the definition of variable elimination. We can then bound $G^t_i(X)$ as follows:
\begin{eqnarray*}
G^t_i(X) &\subseteq & (G^t_{i-1}(X) \cup G^t_{i-1}(Y) \cup G^t_{i-1}(\dup{Y})) \setminus \{Y,\dup{Y}\} \\
& & \:\:\:\mbox{(eliminating $\{Y,\dup{Y}\}$ on $G^t$)}\\
 &\subseteq& (G_{i-1}(X) \cup \dup{G_{i-1}(X)} \cup G_{i-1}(Y) \cup \dup{G_{i-1}(Y)}) \setminus \{Y,\dup{Y}\} \\
&  & \:\:\:\mbox{(by inductive hypothesis)}  \\
&=& (G_{i-1}(X) \cup G_{i-1}(Y)) \cup (\dup{G_{i-1}(X)} \cup \dup{G_{i-1}(Y)}) \setminus \{Y,\dup{Y}\}\\
&= & G_i(X) \cup \dup{G_i(X)}.
\qedhere
\end{eqnarray*}
\end{proof}

\begin{proof}[Proof of Theorem~\ref{thm:twin-cls}]
Consider each variable $X$ that is eliminated at step $i$, i.e. $X=\pi(i)$. By Lemma~\ref{lem:neighbor-set}, $\cls^t(X) = G^t_{i-1}(X) \subseteq G_{i-1}(X) \cup \dup{G_{i-1}(X)} = \cls(X) \cup \dup{\cls(X)}$.

We next bound $\cls^t(\dup{X})$ if $X$ is an internal variable, 
\begin{eqnarray*}
{\cls^t}(\dup{X}) &\subseteq& (G^t_{i-1}(\dup{X}) \cup \cls^t(X)) \setminus \{X\} \\
& & \:\:\:\mbox{(eliminating $X$ from $G^t_{i-1}$)}\\
&\subseteq& G_{i-1}(X) \cup \dup{G_{i-1}(X)} \setminus \{X\} \\
& & \:\:\:\mbox{(by Lemma~\ref{lem:neighbor-set})}\\
&\subseteq & \cls(X) \cup \dup{\cls(X)} .
\qedhere
\end{eqnarray*}
\end{proof}

\begin{proof}[Proof of Corollary~\ref{cor:teo-width}]
By Theorem~\ref{thm:twin-cls}, for every base variable $X \in G^t$, $|\cls^t(X)| \leq |\cls(X) \cup \dup{\cls(X)}| \leq 2|\cls(X)|$.
Similarly, for every duplicate variable $\dup{X} \in G^t$, $|\cls^t(\dup{X})| \leq |\cls(X) \cup \dup{\cls(X)}| \leq 2|\cls(X)|$.
So
$w^t = \max_{X \in G^t} {|\cls^t(X)|} - 1
\leq 2\max_{X \in G} {|\cls(X)|} - 1
= 2(w+1)-1
= 2w+1$.
\end{proof}

\begin{proof}[Proof of Corollary~\ref{cor:twin-treewidth}]
Consider an optimal elimination order $\pi$ for base network $G$ with width $w$. By Corollary~\ref{cor:teo-width}, the twin elimination order for $G^t$ with width no more than $2w+1$. It follows that the treewidth of $G^t$ is no more than $2w+1$.
\end{proof}

\subsection*{Proof of Theorem 2}
We first state a key observation from Algorithm 1, which is formulated as Lemma~\ref{lem:inv-edge}. For simplicity, we say a leaf node \emph{hosts variable $X$} if it hosts a family that contains $X$.

\begin{lemma}~\label{lem:inv-edge}
Consider each invariant edge $(i,j)$ in a twin jointree constructed from Algorithm 1. A variable $X$ is hosted in some leaf on the $i$-side of the edge iff $\dup{X}$ is also hosted in some leaf on the $i$-side of the edge.
\end{lemma}
\begin{proof}[Proof]
If $X$ is a root variable, then $X=\dup{X}$ and the lemma holds. Now suppose $X$ is an internal variable. Let $k$ be the leaf node on the $i$-side of the edge that hosts $X$, it suffices to show that $k$ is duplicated into a $\dup{k}$ on the $i$-side that hosts $\dup{X}$. 

Since $X$ is an internal variable, $k$ hosts either the family for $X$ or a family for some child of $X$. In either case, $k$ hosts a family for an internal variable. Moreover, $k$ is contained in some subtree rooted at $u$ whose leaves host only families for internal variables. Since $k$ is on the $i$-side of the invariant edge $(i,j)$, $u$ is also on the $i$-side of the edge. Hence the duplicate leaf $\dup{k}$, which hosts $\dup{X}$, is also on the $i$-side of the edge.

Conversely, suppose $\dup{X}$ is hosted by some leaf $\dup{k}$ on the $i$-side of edge $(i,j)$, then $\dup{k}$ hosts some family for a duplicate variable. $\dup{k}$ is contained in some duplicate subtree rooted at $\dup{u}$ whose leaves only host families for duplicate variables. It follows that the base subtree rooted at $u$ is located on the $i$-side of the edge and contains a base node $k$ hosting $X$.
\end{proof}

We first recall the definition of separators: $X \in \S_{ij}$ if and only if $X$ is hosted on both sides of the edge $(i,j)$. For simplicity, we use $\vars(i,j)$ to denote the variables that appear on the $i$-side of the edge $(i,j)$ in the base jointree. Similarly, we use $\vars^t(i,j)$ and to denote the variables that appear on the $i$-side in the twin jointree. By definition, for each edge $(i,j)$, $\S_{ij} = \vars(i,j) \cap \vars(j,i)$ and $\S^t_{ij} = \vars^t(i,j) \cap \vars^t(j,i)$. Given a jointree and its root, we say that a jointree node $j$ is \emph{above} a jointree node $i$ if $j$ is closer to the root than $i$, and that $j$ is \emph{below} $i$ if $j$ is further from the root than $i$.

\begin{proof}[Proof of Theorem~\ref{thm:twin-sep}]
We derive the separators for each type of edges. WLG, for each edge $(i,j)$, assume that $j$ is above $i$. 
First consider each duplicated edge $(i,j)$, we have $\vars^t(i,j) = \vars(i,j)$ by Algorithm 1. Moreover, $\vars^t(j,i)$ can only contain extra duplicate variables comparing to $\vars(j,i)$. Thus, 
$\S^t_{ij} = \S_{ij}$.

For each duplicate edge $(\dup{i},\dup{j})$, we have  $\vars^t(\dup{i},\dup{j}) = \dup{\vars(i,j)}$ by Algorithm 1. We next show that for each $\dup{X} \in \vars^t(\dup{i},\dup{j})$, $\dup{X} \in \vars^t(\dup{j},\dup{i})$ iff $X \in \vars(j,i)$, which then concludes $\S^t_{\dup{i} \dup{j}} = \dup{\S_{ij}}$. We first show the if-part. Let $u$ be the least common ancestor of $j$ and $\dup{j}$. If $X \in \vars(j,i)$, then $X$ is hosted by some leaf $k$ that appears either below $u$, or above $u$, in the base jointree. Suppose $k$ is below $u$, then $k$'s duplicate $\dup{k}$ hosts $\dup{X}$ on the $\dup{j}$-side in the twin jointree by Algorithm 1, which implies $\dup{X} \in \vars^t(\dup{j},\dup{i})$. Suppose $k$ is above $u$. If $u$ is the root of the jointree, then $X$ is a root variable and $\dup{X} \in \vars^t(\dup{j},\dup{i})$. Otherwise, let $p$ be the parent of $u$, then $(u,p)$ is an invariant edge, and $\dup{X}$ appears on the $p$-side of the edge by Lemma~\ref{lem:inv-edge}. 

Similar argument applies for the only-if part. Suppose $\dup{X} \in \vars^t(\dup{i},\dup{j})$, then $\dup{X}$ is hosted by some duplicate leaf $\dup{k}$ either below or above $u$. If $\dup{k}$ is below $u$, then the duplicated node $k$ is also below $u$. If $\dup{k}$ is above $u$, then the duplicated node $k$ must also appear above $u$ by Lemma~\ref{lem:inv-edge}.

For each invariant edge $(i,j)$, it follows from Lemma~\ref{lem:inv-edge} that $\vars^t(i,j) = \vars(i,j) \cup \dup{\vars(i,j)}$ and $\vars^t(j,i) = \vars(j,i) \cup \dup{\vars(j,i)}$. Hence, $\S^t_{ij} = \S_{ij} \cup \dup{\S_{ij}}$.
\end{proof}

\begin{proof}[Proof of Corollary~\ref{cor:tjt-width}]
Let $T$ be a jointree for $G$, and let $T^t$ be the twin jointree for $G^t$ obtained from Algorithm~1. Let $i$ be a node in jointree $T$. By Theorem~\ref{thm:thin-sep}, for all neighbors $j$ of $i$, $\S^t_{ij} \subseteq \S_{ij} \cup \dup{\S_{ij}}$. So
$\cls^t_i = \bigcup_{j}{\S^t_{ij}} 
\subseteq \bigcup_{j}{\S_{ij} \cup \dup{\S_{ij}}}
= \cls_i \cup \dup{\cls_i}$. So $|\cls^t_i| \leq |\cls_i \cup \dup{\cls_i}| \leq 2|\cls_i|$.
So $w^t = \max_{i \in T^t} {|\cls^t_i|} - 1
\leq 2\max_{i \in T} {|\cls_i|} - 1
= 2(w+1)-1 = 2w+1$.

The number of nodes in the twin jointree is at most twice the number of nodes in the base jointree since Algorithm 1 adds at most one duplicate for each node in the base jointree.
\end{proof}

\subsection*{Proof of Theorem 3}
\begin{proof}[Proof of Theorem~\ref{thm:thin-sep}]
From~\cite{uai/ChenDarwiche22}, a thinned jointree can be obtained by applying a sequence of thinning rules to a base jointree. Let $\Q = \{Q_1,\dots,Q_T\}$ denote the $T$ thinning rules being applied to the base jointree in order. We next construct a thinning sequence $\Q^t$ for the twin jointree, which leads to the thinned separators defined in the Theorem.  

We first define a \emph{parallel thinning step} as simultaneously thinning a functional variable $X$ from the base jointree and thinning $\{X,\dup{X}\}$ from the twin jointree. Consider any thinning rule $Q_i$ that thins $X$ from a separator $\S_{ij}$, the parallel thinning on $\S^t$ is defined as follows:
\begin{itemize}
\item Suppose $(i,j)$ is a duplicated edge, then we thin $X$ from $\S^t_{ij}$ and $\dup{X}$ from $\S^t_{\dup{i}\dup{j}}$
\item Suppose $(i,j)$ is an invariant edge, then we thin $\{X,\dup{X}\}$ from $\S^t_{ij}$
\end{itemize}

First note that the definition of parallel thinning ensures that the relation between $\S$ and $\S^t$ specified in the theorem holds after every parallel thinning step. It remains to show that the parallel thinnings on $\S^t$ are indeed valid.  

Let $\S$, $\S^t$ denote the separators for the base and twin jointree during the parallel thinnings. Consider a parallel thinning step being applied to a duplicated edge, which, by definition, removes $X$ from $\S_{ij}$, $X$ from $\S^t_{ij}$, and $\dup{X}$ from $\S^t_{\dup{i}\dup{j}}$. Suppose the removal of $X$ from $\S_{ij}$ is supported by the first thinning rule, i.e. the edge $(i, j)$ is on the path between two leaf nodes, call them $l$ and $r$, both hosting the family of $X$, and every separator on that path contains $X$. We claim that the removal of $X$ from $\S^t_{ij}$ and the removal of $\dup{X}$ from $\S^t_{\dup{i}\dup{j}}$ are both supported by the first thinning rule. By Algorithm 1, the leaf nodes $\{l,r\}$ host the family of $X$, and the leaf nodes $\{\dup{l}, \dup{r}\}$ host the family of $\dup{X}$ in the twin jointree. Moreover, by the inductive assumption on separators, $X$ appears on every separator between $l$ and $r$ and $\dup{X}$ appears on every separator between $\dup{l}$ and $\dup{r}$. This is based on an observation that the path from $l$ to $r$ can be divided into three sub-paths: a sub-path consisting of only duplicated edges $(l=p_1,\dots,p_s)$, a sub-path consisting of only invariant edges $(p_s,\dots,p_m)$, and a sub-path consisting of only duplicated edges $(p_m,\dots,p_n=r)$, where $1 < s \leq m < n$\footnote{Note that we do not preclude the possibility of having no invariant edge on the path.}. Given the path from $l$ to $r$, we can then express the path from $\dup{l}$ to $\dup{r}$ as three sub-paths as well: a sub-path consisting of only duplicate edges $(\dup{l}=\dup{p_1},\dots,p_s)$, a sub-path consisting of only invariant edges $(p_s,\dots,p_m)$, and a sub-path consisting of only duplicate edges $(p_m,\dots,\dup{p_n}=\dup{r})$. For each duplicated edge $(i,j)$ on the sub-path from $l$ to $p_s$ or on the sub-path from $p_m$ to $r$, $X \in \S_{ij}$ iff $\dup{X} \in \S_{\dup{i}\dup{j}}$.  For each invariant edge $(i,j)$ on the sub-path from $p_s$ to $p_m$, $X \in \S_{ij}$ iff $\dup{X} \in \S_{ij}$. 

Suppose the removal of $X$ from $\S_{ij}$ is supported by the second thinning rule, i.e. no other separators $\S_{ik}$ contains $X$, or no other separators $\S_{kj}$ contains $X$. Then the removal of $X$ from $\S^t_{ij}$ and the removal of $\dup{X}$ from $\S^t_{\dup{i}\dup{j}}$ are both supported by the second thinning rule due to the inductive assumption on separators.

Consider now a parallel thinning step being applied to an invariant edge which removes $X$ from $\S_{ij}$ and $\{X,\dup{X}\}$ from $\S^t_{ij}$. Suppose the removal of $X$ from $\S_{ij}$ is supported by the first thinning rule, where the edge $(i, j)$ is on the path between two leaf nodes $l$ and $r$ both hosting the family of $X$ and every separator on that path contains $X$. Again, by the inductive assumption on separators, $\dup{X}$ appears in all separators on the path between $\dup{l}$ and $\dup{r}$, which also includes the invariant edge $(i,j)$. Hence, $\dup{X}$ can also be removed from $\S^t_{ij}$ using the first thinning rule. Suppose the removal of $X$ from $\S_{ij}$ is supported by the second thinning rule, then we can apply the second thinning rule to remove $\{X,\dup{X}\}$ from $\S^t_{ij}$ due to the inductive assumption.  
\end{proof}

\subsection*{Proof of Theorem 4}
We first extend the notion of twin graph sequence to \emph{$N$-world graph sequence},  denoted as $G^N_1,G^N_2,\dots,G^N_n$, where $G^N_1$ is the moral graph for $G^N$, and $G^N_{i+1}$ is the result of eliminating $\{\ndup{\pi(i)}{1},\dots,\ndup{\pi(i)}{N}\}$ from $G^N_{i}$. For each variable $\ndup{X}{j}$ ($j \in \{1,\dots,N\}$) in $G^N_i$, let $G^N_i(\ndup{X}{j})$ be the set consisting of $\ndup{X}{j}$ and its neighbors in $G^N_i$. For a set of variables $\X$, we use $\X^j$ to denote $\{X^j | X \in \X\}$.

\begin{lemma}\label{lem:n-twin-neighbor}
Consider a base network $G$ and its $N$-world network $G^N$. Let $\pi$ be an elimination order for $G$. Then for each variable $\ndup{X}{j}$ ($j \in \{1,\dots,N\}$) in $G^N_i$, $G^N_i(\ndup{X}{j}) \subseteq \bigcup_{k=1}^N \ndup{G_i(X)}{k}$.
\end{lemma}
\begin{proof}[Proof]
We will prove this by induction on $G^N_i$. The statement holds initially for $G^N_1$ by the definition of $N$-world networks. Suppose the statement holds for $G^N_{i-1}$, i.e. $G^N_{i-1}(\ndup{X}{j}) \subseteq \bigcup_{k=1}^N \ndup{G_{i-1}(X)}{k}$ for every variable $\ndup{X}{j} \in G^N_{i-1}$, we need to show the statement holds for $G^N_i$.

Let $Y=\pi(i-1)$. Suppose we eliminate variables $\{\ndup{Y}{k}\}_{k=1}^N$ from $G^N$. $G^N_i(X)$ is affected by the elimination of $\{\ndup{Y}{k}\}_{k=1}^N$ iff $X$ is a neighbor of $\{\ndup{Y}{k}\}_{k=1}^N$ in $G^N_i$. Moreover, by induction, $X$ is a neighbor of $\{\ndup{Y}{k}\}_{k=1}^N$ in $G^N_i$ only if $X$ is a neighbor of $Y$ in $G_i$.

When $X$ is not a neighbor of $Y$, then $G^N_i(\ndup{X}{j}) = G^N_{i-1}(\ndup{X}{j})$ for all $j=1,2,\dots,N$ and $G_i(X) = G_{i-1}(X)$, so the statement holds.
When $X$ is a neighbor of $Y$, for all $j=1,2,\dots,N$, we can then bound $G^N_i(\ndup{X}{j})$ as follows:
\begin{eqnarray*}
G^N_i(\ndup{X}{j}) &\subseteq& G^N_{i-1}(\ndup{X}{j}) \cup \bigcup_{k=1}^N (G^N_{i-1}(\ndup{Y}{k}))  \setminus \{\ndup{Y}{k}\}_{k=1}^{N} \\
&\subseteq& (\bigcup_{k=1}^N \ndup{G_{i-1}(X)}{k}) \cup (\bigcup_{k=1}^N \ndup{G_{i-1}(Y)}{k}) \setminus \{\ndup{Y}{k}\}_{k=1}^{N} \\
&& \:\:\:\mbox{(by induction hypothesis)} \\
&=& \bigcup_{k=1}^N(\ndup{G_{i-1}(X)}{k} \cup \ndup{G_{i-1}(Y)}{k} \setminus \{\ndup{Y}{k}\})\\
&=& \bigcup_{k=1}^N \ndup{G_i(X)}{k}.
\qedhere
\end{eqnarray*}
\end{proof}

\begin{lemma}\label{lem:n-twin-cls}
Let $G^N$ be an $N$-world network and let $j \in \{1,\dots,N\}$ be a positive integer. Then $\cls^N(\ndup{X}{j}) \subseteq \bigcup_{k=1}^N \ndup{\cls(X)}{k}$.
\end{lemma}
\begin{proof}[Proof]
Let $X=\pi(i)$.
By Lemma~\ref{lem:n-twin-neighbor}, $\cls^N(\ndup{X}{1}) = G^N_i(\ndup{X}{1}) \subseteq \bigcup_{k=1}^N \ndup{G_i(X)}{k} = \bigcup_{k=1}^N \ndup{\cls(X)}{k}$. We next bound $\C^N(X^j)$ where $j \in \{2,\dots,N\}$ by induction. For each $j \in \{2,\dots,N\}$, assume $\cls^N(\ndup{X}{h}) \subseteq \bigcup_{k=1}^N \ndup{\cls(X)}{k}$ for $h=1,2,\dots,j-1$, then 
\begin{eqnarray*}
\cls^N(\ndup{X}{j}) &\subseteq& G^N_i(\ndup{X}{j}) \cup \bigcup_{k=1}^{j-1} \cls^N(\ndup{X}{k}) \:\:\:\mbox{(by VE definition)} \\
&\subseteq& (\bigcup_{k=1}^N \ndup{G_i(X)}{k}) \cup (\bigcup_{k=1}^N \ndup{\cls(X)}{k})\:\:\:\mbox{(by inductive hypothesis)} \\
&\subseteq& \bigcup_{k=1}^N \ndup{\cls(X)}{k}
\qedhere
\end{eqnarray*}
\end{proof}

\begin{proof}[Proof of Theorem~\ref{thm:ntwin-tw}]
By Lemma~\ref{lem:n-twin-cls}, for all $\ndup{X}{j} \in G^N$, $|\cls(\ndup{X}{j})| \leq \sum_{k=1}^{N} |\ndup{\cls(X)}{k}| = N|\cls(X)|$.
Therefore, the width $w^N$ of $\pi^N$ is
\begin{eqnarray*}
w^t &=& \max_X {\cls^N(X)} - 1 
\leq N\max_X {\cls(X)} - 1
= N(w+1)-1\\
\end{eqnarray*}

For causal treewidth, we can extend Algorithm 1 to construct $N$-world jointrees by making $N-1$ duplicates for each duplicated subtree, and construct $N$-world thinned jointrees by extending Theorem 3. By analogous arguments, we can show that $w^t = N(w+1)-1$.
\end{proof}

\begin{figure}[tb]
\centering
\begin{subfigure}[b]{0.48\columnwidth}
\centering
\begin{tikzpicture}[->,>=stealth,shorten >=1pt,auto,scale=0.45,transform shape]
\node[state,thin,font=\huge] (A) at (0,1) {$A$};
\node[state,thin,font=\huge] (B) at (0,-1) {$B$};
\node[state,thin,font=\huge] (C) at (-2,2) {$C$};
\node[state,thin,font=\huge] (D) at (-2,0) {$D$};
\node[state,thin,font=\huge] (E) at (-2,-2) {$E$};
\node[state,thin,font=\huge] (F) at (-4,0) {$F$};

\path (A) edge (C);
\path (A) edge (D);
\path (B) edge (E);
\path (D) edge (E);
\path (C) edge (D);
\path (C) edge (F);
\path (E) edge (F);
\end{tikzpicture}
\caption{base network, width of \(\pi\) is $w=2$}
\label{fig:tb-bn}
\end{subfigure}
\begin{subfigure}[b]{0.48\columnwidth}
\centering
\begin{tikzpicture}[->,>=stealth,shorten >=1pt,auto,scale=0.45,transform shape]
\node[state,thin,font=\huge] (A) at (0,1) {$A$};
\node[state,thin,font=\huge] (B) at (0,-1) {$B$};
\node[state,thin,font=\huge] (C) at (-2,2) {$C$};
\node[state,thin,font=\huge] (D) at (-2,0) {$D$};
\node[state,thin,font=\huge] (E) at (-2,-2) {$E$};
\node[state,thin,font=\huge] (F) at (-4,0) {$F$};
\node[state,thin,font=\huge] (C1) at (2,2.2) {$\dup{C}$};
\node[state,thin,font=\huge] (D1) at (2,0) {$\dup{D}$};
\node[state,thin,font=\huge] (E1) at (2,-2.2) {$\dup{E}$};
\node[state,thin,font=\huge] (F1) at (4.2,0) {$\dup{F}$};

\path (A) edge (C);
\path (A) edge (D);
\path (B) edge (E);
\path (D) edge (E);
\path (C) edge (D);
\path (C) edge (F);
\path (E) edge (F);
\path (A) edge (C1);
\path (A) edge (D1);
\path (B) edge (E1);
\path (D1) edge (E1);
\path (C1) edge (D1);
\path (C1) edge (F1);
\path (E1) edge (F1);
\end{tikzpicture}
\caption{twin network, width of \(\pi^t\) is $w^t=5$}
\label{fig:tb-tn}
\end{subfigure}
\caption{The elimination order $\pi$ is $(A,B,F,D,C,E)$ and
the twin elimination order
$\pi^t$ is $(A,B,F,\dup{F},D,\dup{D},C,\dup{C},E,\dup{E})$.
We have $w^t=2w+1$.}
\end{figure}
\begin{figure}[tb]
\centering
\begin{subfigure}[b]{0.43\columnwidth}
\centering
\includegraphics[width=0.5\columnwidth]{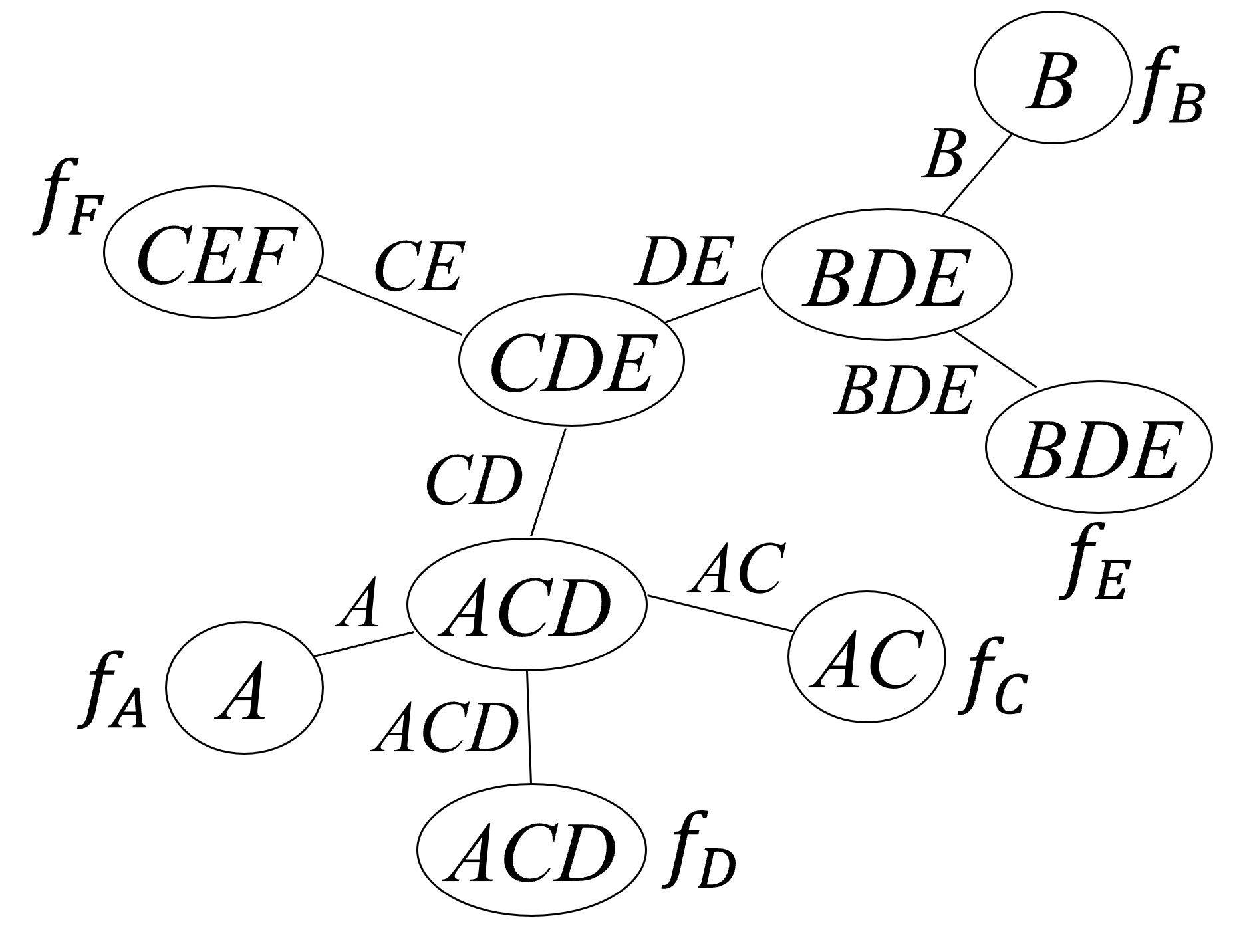}
\caption{base jointree, width $w=2$}
\label{fig:tb-jt}
\end{subfigure}
\begin{subfigure}[b]{0.53\columnwidth}
\centering
\includegraphics[width=0.6\columnwidth]{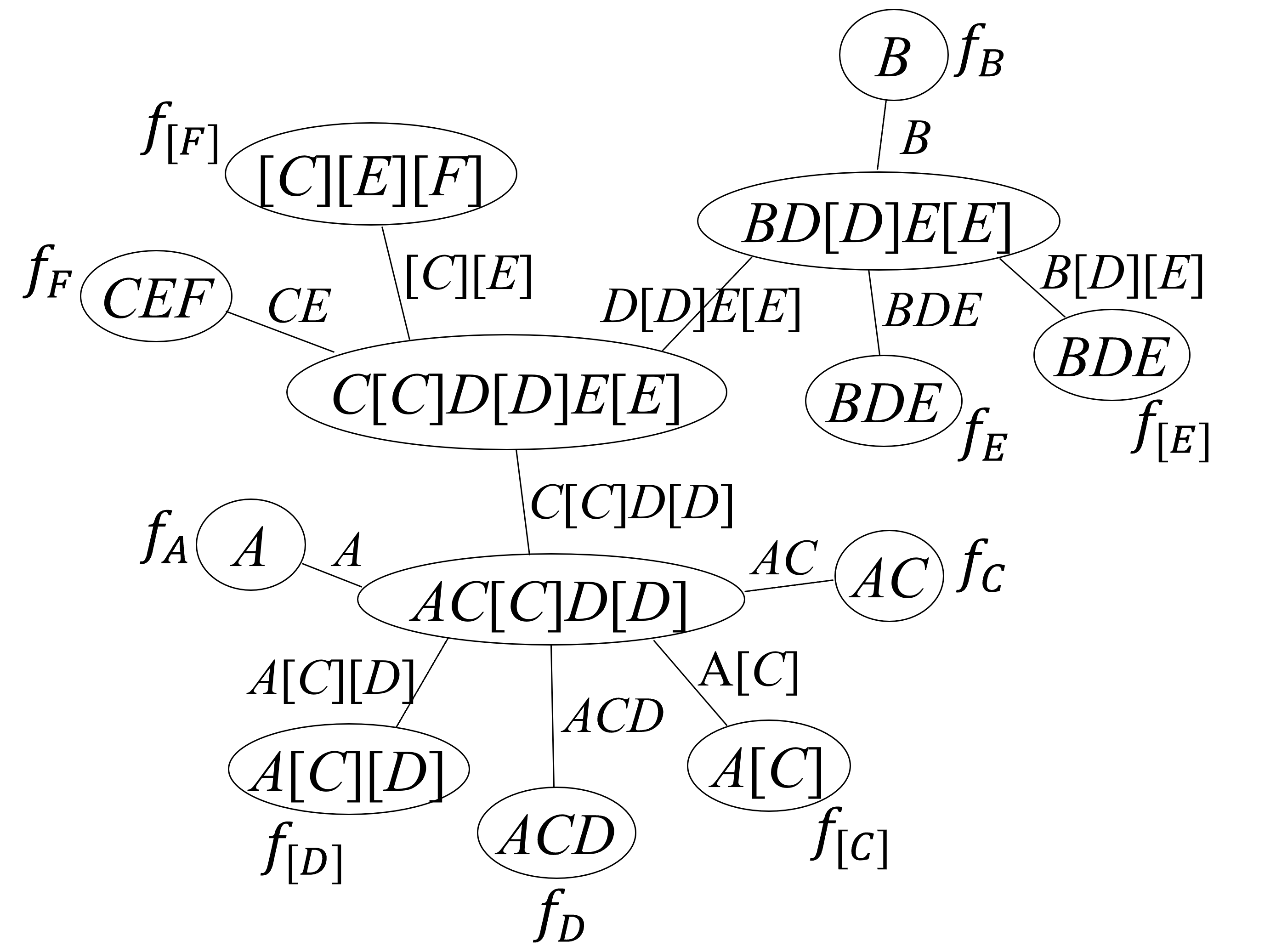}
\caption{twin jointree, width $w^t=5$}
\label{fig:tb-tjt}
\end{subfigure}
\caption{A base jointree and its twin jointree, where $w^t=2w+1$.}
\end{figure}
\begin{figure}[tb]
\centering
\begin{subfigure}[b]{0.48\columnwidth}
\centering
\begin{tikzpicture}[->,>=stealth,shorten >=1pt,auto,scale=0.45,transform shape]
\node[state,thin,font=\huge] (A) at (0,3) {$A$};
\node[state,thin,font=\huge] (B) at (0,1.5) {$B$};
\node[state,thin,font=\huge] (C) at (0,0) {$C$};
\node[state,thin,font=\huge] (D) at (0,-1.5) {$D$};
\node[state,thin,font=\huge] (E) at (-5,3) {$E$};
\node[state,thin,font=\huge] (F) at (-3.5,1.5) {$F$};
\node[state,thin,font=\huge] (G) at (-2,0) {$G$};
\node[state,thin,font=\huge] (H) at (-3.5,-1.5) {$H$};
\node[state,thin,font=\huge] (I) at (-5,0) {$I$};

\path (A) edge (E);
\path (B) edge (F);
\path (C) edge (G);
\path (D) edge (H);
\path (E) edge (F);
\path (F) edge (G);
\path (G) edge (H);
\path (H) edge (I);
\path (E) edge (I);
\end{tikzpicture}
\caption{base network, treewidth $w=2$}
\label{fig:tw-bn}
\end{subfigure}
\begin{subfigure}[b]{0.48\columnwidth}
\centering
\begin{tikzpicture}[->,>=stealth,shorten >=1pt,auto,scale=0.45,transform shape]
\node[state,thin,font=\huge] (A) at (0,3) {$A$};
\node[state,thin,font=\huge] (B) at (0,1.5) {$B$};
\node[state,thin,font=\huge] (C) at (0,0) {$C$};
\node[state,thin,font=\huge] (D) at (0,-1.5) {$D$};
\node[state,thin,font=\huge] (E) at (-5,3) {$E$};
\node[state,thin,font=\huge] (F) at (-3.5,1.5) {$F$};
\node[state,thin,font=\huge] (G) at (-2,0) {$G$};
\node[state,thin,font=\huge] (H) at (-3.5,-1.5) {$H$};
\node[state,thin,font=\huge] (I) at (-5,0) {$I$};
\node[state,thin,font=\huge] (E1) at (5,3) {$\dup{E}$};
\node[state,thin,font=\huge] (F1) at (3.5,1.5) {$\dup{F}$};
\node[state,thin,font=\huge] (G1) at (2,0) {$\dup{G}$};
\node[state,thin,font=\huge] (H1) at (3.5,-1.5) {$\dup{H}$};
\node[state,thin,font=\huge] (I1) at (5,0) {$\dup{I}$};

\path (A) edge (E);
\path (B) edge (F);
\path (C) edge (G);
\path (D) edge (H);
\path (E) edge (F);
\path (F) edge (G);
\path (G) edge (H);
\path (H) edge (I);
\path (E) edge (I);
\path (A) edge (E1);
\path (B) edge (F1);
\path (C) edge (G1);
\path (D) edge (H1);
\path (E1) edge (F1);
\path (F1) edge (G1);
\path (G1) edge (H1);
\path (H1) edge (I1);
\path (E1) edge (I1);
\end{tikzpicture}
\caption{twin network, treewidth $w^t=4$}
\label{fig:tw-tn}
\end{subfigure}
\caption{A base network and its corresponding twin network, where $w^t=2w$.}
\end{figure}
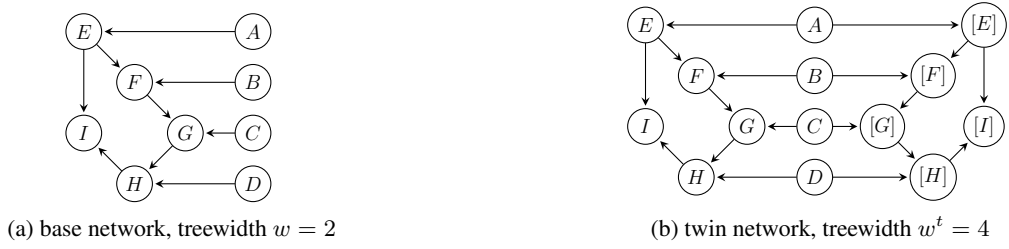

\section{Tightness of Bounds}
\label{app:tight}
As we claimed in the paper, Corollaries~$1$ and~$3$ provide tight bounds for the widths of twin elimination orders defined by Definition~$2$, and twin jointrees constructed by Algorithm~1, respectively. Moreover, we provide a concrete example where a base
network has treewidth \(w\) and its twin network has treewidth \(2w\), which suggests
that our bound on treewidth in Corollary~$2$ may also be tight (we found 
these examples through exhaustive search which we were able to do on small examples
only, given the allotted computational resources).

For the tightness of Corollary~$1$, 
consider the base network $G$ shown in Figure~\ref{fig:tb-bn}. The elimination order $\pi=(A,B,F,D,C,E)$ has width of $2.$ The twin network $G^t$ for $G$ is shown in Figure~\ref{fig:tb-tn}. By Definition~$2$, the twin elimination order for $G^t$ is $\pi^t=(A,B,F,\dup{F},D,\dup{D},C,\dup{C},E,\dup{E})$ which has a width of $5.$

For the tightness of Corollary~$3$, 
Figure~\ref{fig:tb-jt} shows a base jointree for the base network 
\(G\). This base jointree was constructed using elimination order \(\pi\)
and it has a width of \(2\).
Figure~\ref{fig:tb-tjt} shows a twin jointree constructed by Algorithm~1
which has a width of $5.$

For our last example, 
Figure~\ref{fig:tw-bn} shows a base network with a treewidth of $2.$ 
Figure~\ref{fig:tw-tn} shows the corresponding twin network which has a treewidth of $4.$
This shows that if our treewidth bound in Corollary~$2$ is not tight, then it
is off by at most~\(1\).

\section{Additional Figures}
\label{app:fig}
Figure~\ref{fig:thinned-jt} shows a thinned jointree for the base network in Figure~\ref{fig:base-nett}. Figure~\ref{fig:twin-thinned-jt} shows the 
corresponding thinned, twin jointree constructed by Algorithm~1 and annotated with the thinned separators (and clusters) as given by Theorem~3. 
The jointree width is preserved in this case.

\begin{figure*}[h]
\centering
\begin{subfigure}[b]{0.49\columnwidth}
\centering
\begin{tikzpicture}[->,>=stealth,shorten >=1pt,auto,scale=0.45,transform shape]
\node[state,thin,font=\huge] (U) at (0,0) {$U$};
\node[state,thin,font=\huge] (X) at (-2,0) {$X$};
\node[state,thin,font=\huge] (Y) at (2,0) {$Y$};
\node[state,thin,font=\huge] (A) at (-1,-1.5) {$A$};
\node[state,thin,font=\huge] (B) at (1,-1.5) {$B$};
\node[state,thin,font=\huge] (S) at (-2,-3) {$S$};
\node[state,thin,font=\huge] (C) at (2,-3) {$C$};
\path (U) edge (A);
\path (U) edge (B);
\path (A) edge (S);
\path (B) edge (S);
\path (X) edge (S);
\path (A) edge (C);
\path (B) edge (C);
\path (Y) edge (C);
\end{tikzpicture}
\caption{base network}
\label{fig:base-nett}
\end{subfigure}
\hfill
\begin{subfigure}[b]{0.49\columnwidth}
\centering
\includegraphics[width=0.5\columnwidth]{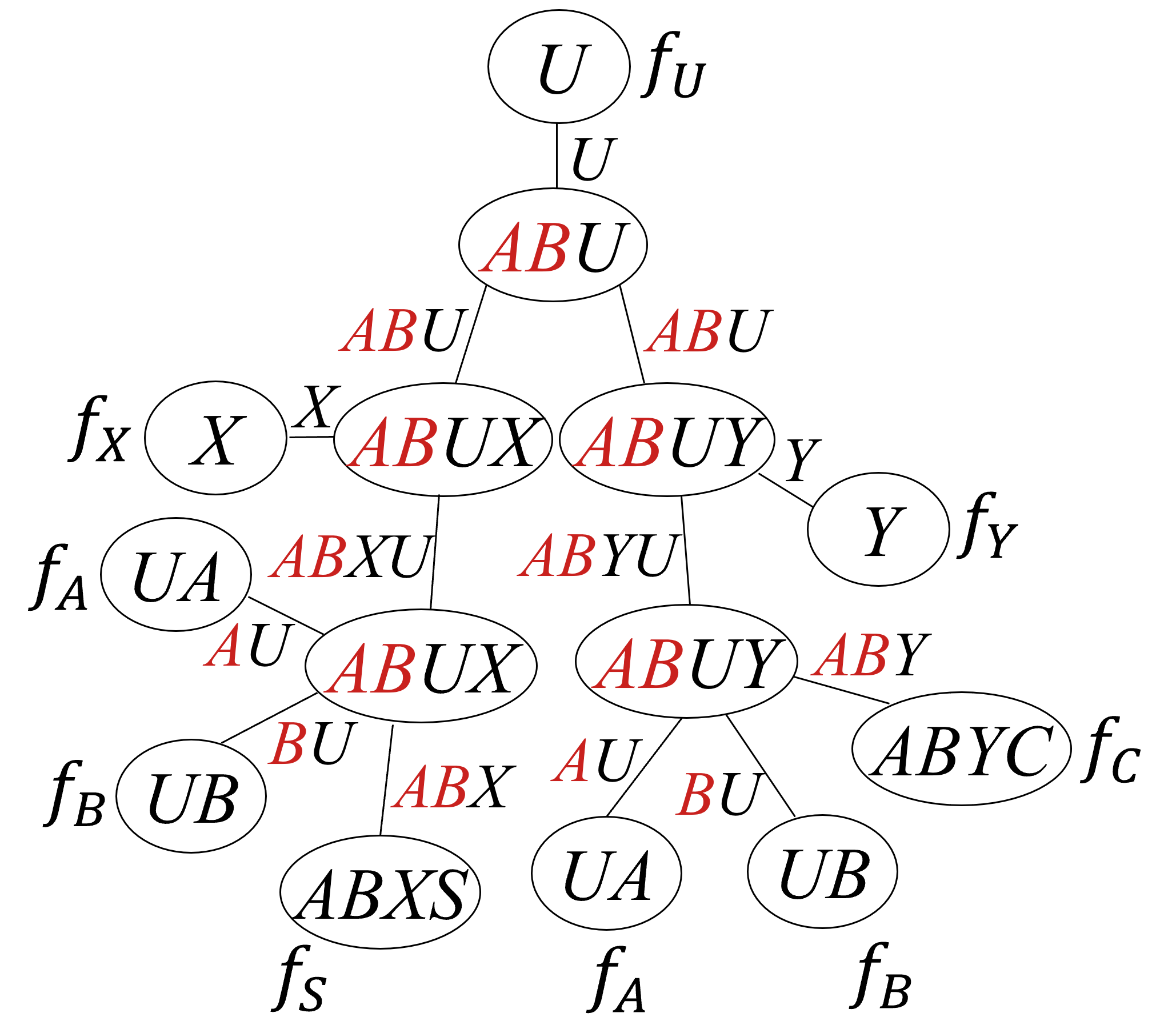}
\caption{thinned base jointreen (width \(3\))}
\label{fig:thinned-jt}
\end{subfigure}
\begin{subfigure}[b]{1.0\columnwidth}
\centering
\includegraphics[width=0.5\columnwidth]{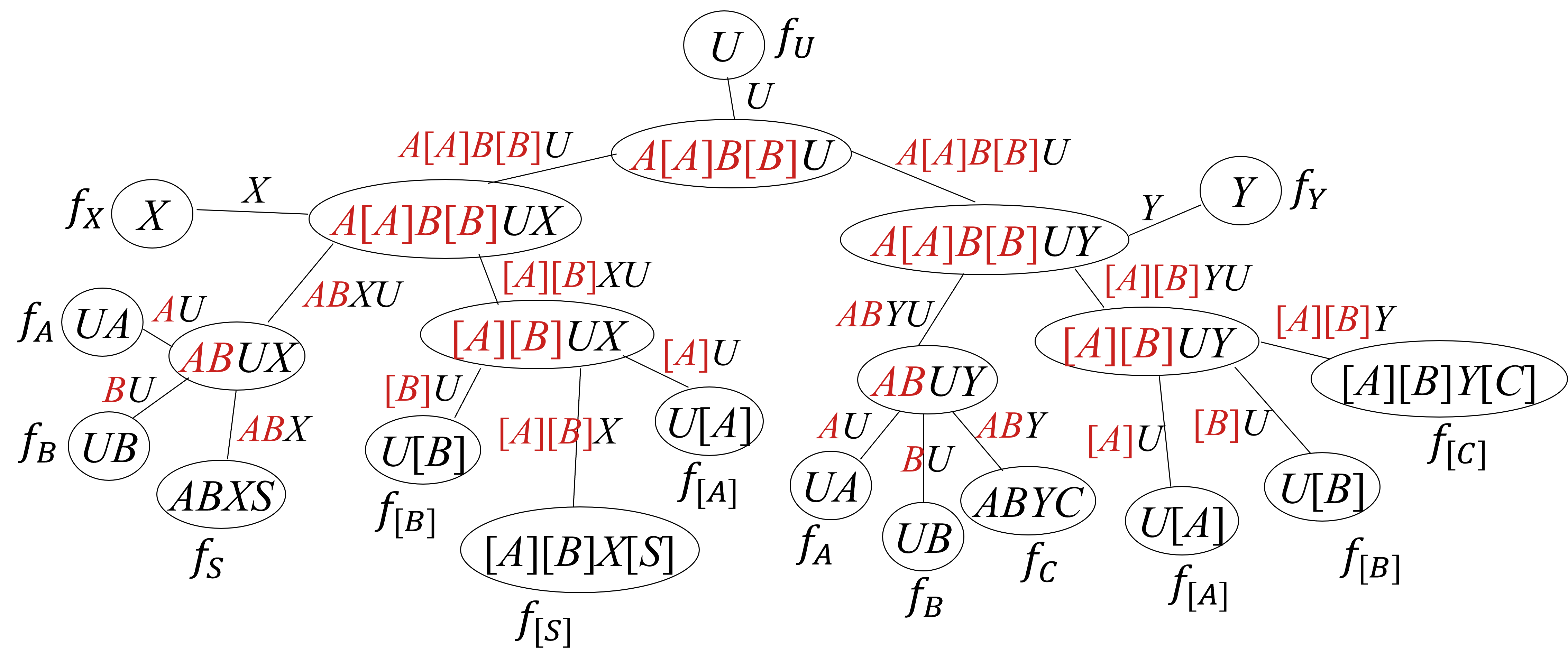}
\caption{thinned twin jointree (width \(3\))}
\label{fig:twin-thinned-jt}
\end{subfigure}
\caption{Illustrating the construction of a thinned, twin jointree
from a thinned, base jointree. }
\end{figure*}
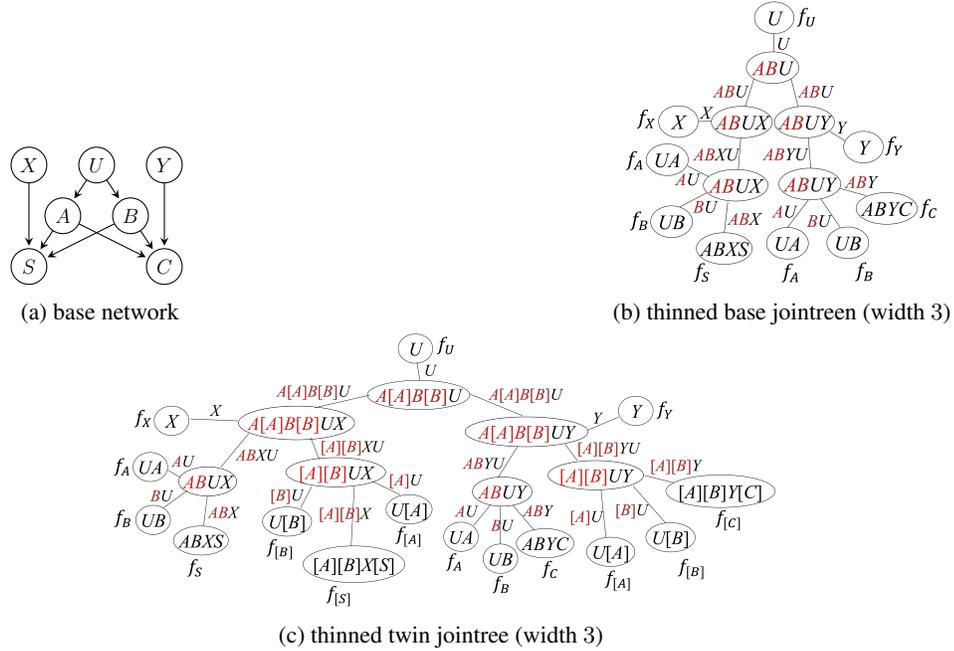

\section{Generalized $N$-World Networks}
\label{app:gn-world-net}
We show that the bound on treewidth also applies to a
more general version of $N$-world networks,
called \emph{generalized $N$-world networks},
that allow any selected, subset of nodes to be duplicated and also allow certain edges
between duplicates of the same variable.
\begin{definition}
\label{def:g-n-world-net}
Let $G$ be a base network
and let $\X$ be any subset of its nodes.
A \underline{generalized $N$-world network}
is constructed as follows. 
We replace each node $X\in\X$ with $N$ duplicates denoted as 
$X^1,\dots,X^N$. 
For each parent $P$ of $X$,
if $P$ is not in $\X$, then make $P$ a parent of 
$X^1,\dots,X^N$;
otherwise, make $P^i$ a parent of $X^i$ 
for each $i\in 1,\dots,N$.
Moreover, for each pair of duplicates $X^i$
and $X^j$ where $1 \leq i < j \leq N$,
an edge may be added from $X^i$ to $X^j$.
\end{definition}

Figure~\ref{fig:gn-tn} shows a generalized $3$-world network for
the base network in Figure~\ref{fig:gn-bn}.
The notion of generalized $N$-world networks is significant
as it covers a vast subclass of dynamic Bayesian networks,
which is widely used for temporal reasoning.

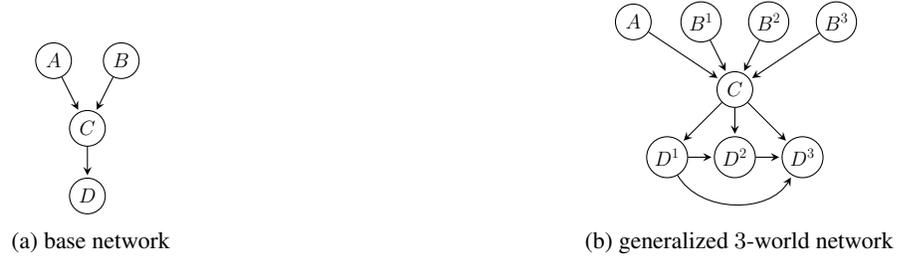
\begin{figure}[h]
\centering
\begin{subfigure}[b]{0.48\columnwidth}
\centering
\begin{tikzpicture}[->,>=stealth,shorten >=1pt,auto,scale=0.45,transform shape]
\node[state,thin,font=\huge] (A) at (0,0) {$A$};
\node[state,thin,font=\huge] (B) at (2,0) {$B$};
\node[state,thin,font=\huge] (C) at (1,-2) {$C$};
\node[state,thin,font=\huge] (D) at (1,-4) {$D$};

\path (A) edge (C);
\path (B) edge (C);
\path (C) edge (D);
\end{tikzpicture}
\caption{base network}
\label{fig:gn-bn}
\end{subfigure}
\begin{subfigure}[b]{0.48\columnwidth}
\centering
\begin{tikzpicture}[->,>=stealth,shorten >=1pt,auto,scale=0.45,transform shape]
\node[state,thin,font=\huge] (A) at (0,0) {$A$};
\node[state,thin,font=\huge] (B1) at (2,0) {$B^1$};
\node[state,thin,font=\huge] (B2) at (4,0) {$B^2$};
\node[state,thin,font=\huge] (B3) at (6,0) {$B^3$};
\node[state,thin,font=\huge] (C) at (3,-2) {$C$};
\node[state,thin,font=\huge] (D1) at (1,-4) {$D^1$};
\node[state,thin,font=\huge] (D2) at (3,-4) {$D^2$};
\node[state,thin,font=\huge] (D3) at (5,-4) {$D^3$};

\path (A) edge (C);
\path (B1) edge (C);
\path (C) edge (D1);
\path (B2) edge (C);
\path (C) edge (D2);
\path (B3) edge (C);
\path (C) edge (D3);

\path (D1) edge (D2);
\path (D2) edge (D3);
\path (D1) edge[bend right=60] (D3);
\end{tikzpicture}
\caption{generalized $3$-world network}
\label{fig:gn-tn}
\end{subfigure}
\caption{A base network and a corresponding generalized $3$-world
network. Nodes $B,D$ were duplicated and edges $D^1 \rightarrow D^2$, $D^2 \rightarrow D^3$, $D^1 \rightarrow D^3$ were added.}
\end{figure}

The following result shows that our bound still applies
to generalized $N$-world networks.
It is worth mentioning that the bound also applies 
for any subgraphs of a generalized $N$-world network,
as the treewidth of any subgraph is no greater than
the treewidth of the original network.
\begin{theorem}
Let $w$ and $w^t$ be the treewidths of a base network
and its generalized $N$-world network, then
$w^t \leq N(w+1)-1$.
\end{theorem}
\begin{proof}
The proof is similar to the one for Theorem~\ref{thm:ntwin-tw}. We are assuming here that the duplicated nodes of $G$ are $\Y$.
We maintain the following invariant:
for each $i \in 1 \dots, n$ and each
node $X^j$ ($j \in 1 \dots, N$) in $G^N_i$,
$G^N_i(X^j) \subseteq \bigcup_{k=1}^N G_i(X)^k$.
To show the invariant holds initially,
observe that the construction of a generalized $N$-world network
involves two steps:
duplicating the nodes in $\Y$ and adding edges
between the duplicates of $\Y$.
After the first step, the invariant holds
since an edge is added between two nodes in $G^N$
only if the parent-child relationship exists in the base
network.
We next show that the invariant still holds after the 
second step.
Each time we add an edge from one duplicate $Y^j$ to another
duplicate $Y^l$ ($1 \leq j < l \leq N$), 
$Y^j$ becomes a neighbor of $Y^l$ and the parents of $Y^l$.
Since $Y^j$, $Y^l$ and the parents of $Y^l$ all belong to
the set $\bigcup_{k=1}^N G_1(Y)^k$,
$G^N_1(Y^j)$ and $G^N_1(Y^l)$ are still subsets of $\bigcup_{k=1}^N G_1(Y)^k$ after the edge is added.
Similarly, for each parent $P^s$ of $Y^l$ where $1 \leq s \leq N$,
since $Y^j$ already belongs to the set $\bigcup_{k=1}^N G_1(P)^k$,
$G^N_1(P^s)$ remains a subset of $\bigcup_{k=1}^N G_1(P)^k$
after the edge is added.
It follows inductively 
that the invariant still holds after all edges are added.

We next show that the invariant holds after each
elimination step.
This can be done with a same inductive argument as Lemma~\ref{lem:n-twin-neighbor}.
Suppose the invariant holds
in $G^N_{i-1}$ 
and we want to eliminate $\{Y^k\}_{k=1}^N$,
we consider how the eliminations affect
the neighbors of each node $X^j$ in $G^N_{i-1}$.
Suppose $Y$ is not in $G_{i-1}(X)$,
then by the inductive assumption, none of $\{Y^k\}_{k=1}^N$
are in $G^N_{i-1}(X^j)$. Thus $G_{i}(X) = G_{i-1}(X)$
and $G^N_{i}(X^j)=G^N_{i-1}(X^j)$.
Suppose $Y$ is in $G_{i-1}(X)$,
then $G_{i}(X) = G_{i-1}(X) \cup G_{i-1}(Y) \setminus \{Y\}$
and $G^N_{i}(X^j) \subseteq 
(\bigcup_{k=1}^N \ndup{G_{i-1}(X)}{k}) \cup (\bigcup_{k=1}^N \ndup{G_{i-1}(Y)}{k}) \setminus \{\ndup{Y}{k}\}_{k=1}^{N}
=\bigcup_{k=1}^N \ndup{G_i(X)}{k}$. The inductive assumption
holds for both cases.

We finally show that the cluster formed for each $Y^j$ 
($1 \leq j \leq N$),
denoted $\C^N(Y^j)$, is a subset of 
$\bigcup_{k=1}^N \ndup{\C(Y)}{k}$, where $\C(Y) = G_{i-1}(Y)$.
This is similar to Lemma~\ref{lem:n-twin-cls}.
For each $Y^j$, by the definition of variable elimination,
the cluster is bounded as
$\C^N(Y^j) = \bigcup_{k=1}^j G^N_{i-1}(Y^k) \subseteq \bigcup_{k=1}^N \ndup{G_{i-1}(Y)}{k} = \bigcup_{k=1}^N \C(Y)^{k}$.
We can then conclude the bound on widths and treewidths.
\end{proof}

\section{More on Experiments}
\label{app:exp}

Our experiments were run on 2.60GHz Intel Xeon E5-2670 CPU with 256 GB of memory. 

In the main paper, we plotted the jointree widths for random networks with a maximum number
of parents ($p=5$). Here we show the complete experimental results for the random networks with $p = 3,5,7$ generated according to the method of~\cite{DarwicheECAI20b}
which we discussed in the main paper. 
We record the widths and normalized widths for each twin jointrees. Recall that if a jointree has clusters \(\C_1, \ldots, \C_n\), then the normalized width is
\(\log_2 \sum_{i=1}^n 2^{|\C_i|}\).
Table~\ref{tab:exp1} shows the complete results for rNET and Table~\ref{tab:exp2} shows the complete results for rSCM.

In addition, we report results on random networks generated by a second method proposed in~\cite{sbia/IdeC02}. Given a number of nodes $n$ and a maximum degree\footnote{The degree of a node is the number of its parents and children. } $d$ for each node, the method generates a random network by repeatedly adding/removing random edges from the current DAG. We refer to these random networks as rNET2. Similarly to what we did for rNET, we then added a unique root as parent for each internal variable in rNET2. We refer to these modified networks as rSCM2. For each combination of $n \in \{50,100,150,200\}$ and $d \in \{5,10,15\}$, we generated $50$ random base networks and reported the average widths and normalized widths for each twin jointree. Table~\ref{tab:exp3} shows the complete results for rNET2 and Table~\ref{tab:exp4} shows the complete results for rSCM2. The patterns in these tables resemble the ones for rNET and rSCM.

\begin{table*}[tb]
\scriptsize
\centering
\def\arraystretch{1.25}
\begin{tabular}{| *{9}{c|} | *{6}{c|}}
\hline
\multirow{2}{*}{\makecell{num\\ vars}} & \multirow{2}{*}{\makecell{max\\ pars}} & \multirow{2}{*}{stats} & \multicolumn{2}{c|}{\makecell{BASE-\\MF}} & \multicolumn{2}{c|}{\makecell{TWIN-\\ALG1}} & \multicolumn{2}{c||}{\makecell{TWIN-\\MF}} & \multicolumn{2}{c|}{\makecell{BASE-\\MF-RLS}} & \multicolumn{2}{c|}{\makecell{TWIN-\\THM3}} & \multicolumn{2}{c|}{\makecell{TWIN-\\MF-RLS}}\\
\cline{4-15}
 & & & wd & nwd & wd & nwd & wd & nwd & wd & nwd & wd & nwd &  wd & nwd \\
\hline

\multirow{6}{*}{50} 
& \multirow{2}{*}{3} & mean & 7.5 & 11.5 & 14.4 & 17.1 & 10.0 & 14.0 & 5.2 & 11.9 & 7.6 & 13.2 & 5.6 & 13.0\\
\cline{3-15}
& & std & 1.4 & 0.9 & 2.8 & 2.4 & 2.0 & 1.4 & 1.0 & 0.7 & 1.2 & 0.7 & 1.2 & 0.7\\
\cline{2-15}
&  \multirow{2}{*}{5} & mean & 14.3 & 17.4 & 26.8 & 29.2 & 15.9 & 19.6 & 7.2 & 15.0 & 10.0 & 16.2 & 7.3 & 16.0\\
\cline{3-15}
& & std & 2.0 & 1.7 & 4.0 & 3.8 & 1.9 & 1.7 & 1.2 & 0.8 & 2.1 & 0.9 & 1.1 & 0.8\\
\cline{2-15}
& \multirow{2}{*}{7} & mean & 19.5 & 22.4 & 36.9 & 39.1 & 20.4 & 24.1 & 8.9 & 17.1 & 12.2 & 18.3 & 8.9 & 18.1\\
\cline{3-15}
& & std & 2.0 & 1.9 & 4.1 & 3.9 & 1.8 & 1.7 & 0.8 & 0.7 & 2.4 & 0.8 & 0.9 & 0.7\\
\hline \hline

\multirow{6}{*}{75} 
& \multirow{2}{*}{3} & mean & 10.2 & 13.8 & 18.9 & 21.5 & 14.5 & 17.9 & 6.1 & 12.9 & 8.6 & 14.3 & 6.9 & 14.0\\
\cline{3-15}
& & std & 1.9 & 1.5 & 3.6 & 3.4 & 2.4 & 2.0 & 1.6 & 0.8 & 1.9 & 0.9 & 1.9 & 0.8\\
\cline{2-15}
& \multirow{2}{*}{5} & mean & 21.3 & 23.9 & 39.1 & 41.1 & 23.9 & 27.2 & 8.4 & 16.4 & 11.9 & 17.8 & 9.1 & 17.5\\
\cline{3-15}
& & std & 2.3 & 2.2 & 5.1 & 4.9 & 3.0 & 2.7 & 1.6 & 0.8 & 2.2 & 1.0 & 1.8 & 0.8\\
\cline{2-15}
& \multirow{2}{*}{7} & mean & 28.8 & 31.4 & 54.0 & 56.0 & 30.3 & 33.7 & 11.0 & 18.6 & 16.4 & 21.2 & 11.7 & 19.7 \\
\cline{3-15}
& & std & 2.6 & 2.5 & 5.2 & 5.2 & 2.2 & 2.1 & 1.9 & 0.9 & 4.1 & 2.4 & 2.0 & 0.9\\

\hline \hline

\multirow{6}{*}{100} 
& \multirow{2}{*}{3} & mean & 13.7 & 16.7 & 25.1 & 27.5 & 19.9 & 22.8 & 7.1 & 13.8 & 10.2 & 15.3 & 8.3 & 14.9\\
\cline{3-15}
& & std & 2.4 & 2.1 & 4.7 & 4.4 & 3.3 & 2.8 & 1.7 & 0.8 & 2.2 & 1.2 & 1.9 & 0.9\\
\cline{2-15}
& \multirow{2}{*}{5} & mean & 27.1 & 29.7 & 50.8 & 52.9 & 31.4 & 34.5 & 9.5 & 17.1 & 13.5 & 18.9 & 11.2 & 18.4 \\
\cline{3-15}
& & std & 2.5 & 2.5 & 5.1 & 5.0 & 3.6 & 3.2 & 2.2 & 0.6 & 2.9 & 1.5 & 3.0 & 0.9\\
\cline{2-15}
& \multirow{2}{*}{7} & mean & 38.7 & 40.9 & 72.9 & 74.7 & 40.8 & 43.8 & 14.5 & 20.3 & 22.9 & 26.6 & 15.9 & 21.9 \\
\cline{3-15}
& & std & 3.5 & 3.4 & 6.7 & 6.5 & 3.3 & 3.2 & 3.3 & 1.5 & 7.1 & 5.8 & 3.4 & 1.8\\
\hline \hline

\multirow{6}{*}{125}
& \multirow{2}{*}{3} & mean & 16.7 & 19.6 & 30.8 & 33.0 & 26.9 & 29.3 & 7.9 & 14.6 & 10.5 & 16.1 & 9.2 & 15.9\\
\cline{3-15}
& & std & 2.1 & 1.9 & 4.0 & 3.8 & 3.9 & 3.6 & 2.0 & 0.8 & 2.3 & 1.2 & 2.3 & 1.0\\
\cline{2-15}
& \multirow{2}{*}{5} & mean & 34.9 & 37.3 & 65.1 & 66.9 & 39.1 & 42.1 & 12.1 & 18.3 & 17.2 & 21.4 & 14.8 & 20.3\\
\cline{3-15}
& & std & 3.1 & 2.9 & 5.9 & 5.7 & 3.6 & 3.2 & 2.4 & 1.0 & 3.8 & 2.5 & 3.6 & 1.9\\
\cline{2-15}
& \multirow{2}{*}{7} & mean & 47.9 & 50.1 & 89.8 & 91.7 & 50.6 & 53.5 & 19.7 & 23.7 & 32.5 & 35.1 & 21.7 & 26.1\\
\cline{3-15}
& & std & 3.2 & 3.1 & 6.8 & 6.7 & 3.8 & 3.6 & 4.2 & 3.3 & 8.7 & 8.3 & 4.5 & 3.7\\
\hline \hline

\multirow{6}{*}{150} 
& \multirow{2}{*}{3} & mean & 20.3 & 22.9 & 36.9 & 38.9 & 31.8 & 34.2 & 8.6 & 15.2 & 11.1 & 16.7 & 9.8 & 16.4 \\
\cline{3-15}
& & std & 2.4 & 2.3 & 4.7 & 4.7 & 4.2 & 3.9 & 2.1 & 0.8 & 2.2 & 1.1 & 2.5 & 1.1\\
\cline{2-15}
& \multirow{2}{*}{5} & mean & 41.4 & 43.6 & 76.9 & 78.6 & 49.1 & 51.7 & 14.6 & 19.7 & 21.2 & 24.7 & 18.6 & 22.8\\
\cline{3-15}
& & std & 3.3 & 3.1 & 7.4 & 7.2 & 6.3 & 5.8 & 2.9 & 1.8 & 5.3 & 4.5 & 4.0 & 3.2\\
\cline{2-15}
& \multirow{2}{*}{7} & mean & 56.7 & 58.7 & 106.7 & 108.2 & 60.6 & 63.3 & 22.5 & 26.1 & 36.8 & 39.3 & 25.6 & 29.4\\
\cline{3-15}
& & std & 4.1 & 3.8 & 8.1 & 7.9 & 3.6 & 3.5 & 4.8 & 3.9 & 9.6 & 9.1 & 5.1 & 4.5\\
\hline \hline

\multirow{6}{*}{200} 
& \multirow{2}{*}{3} & mean & 25.8 & 28.2 & 47.2 & 49.1 & 42.9 & 44.8 & 10.7 & 16.4 & 13.6 & 18.3 & 12.3 & 17.9\\
\cline{3-15}
& & std & 2.5 & 2.5 & 4.9 & 4.9 & 5.4 & 5.2 & 2.9 & 1.5 & 2.9 & 1.5 & 3.2 & 1.9\\
\cline{2-15}
& \multirow{2}{*}{5} & mean & 55.1 & 57.2  & 102.4  & 104.0 & 65.7 & 68.0 & 19.9 & 23.8 & 29.5 & 32.3 & 26.0 & 29.4\\
\cline{3-15}
& & std & 3.6 & 3.5 & 7.8 & 7.6 & 7.0 & 6.8 & 4.2 & 3.3 & 8.1 & 7.6 & 6.1 & 5.3\\
\cline{2-14}
& \multirow{2}{*}{7} & mean & 75.3 & 77.4 & 141.5 & 143.1 & 80.4 & 82.7 & 34.3 & 37.1 & 57.5 & 59.8 & 39.0 & 42.3\\
\cline{3-15}
& & std & 4.1 & 3.9 & 8.4 & 8.2 & 4.3 & 4.1 & 4.8 & 4.7 & 10.2 & 10.2 & 5.1 & 4.8\\
\hline \hline

\multirow{6}{*}{250} 
& \multirow{2}{*}{3} & mean & 31.6 & 34.0 & 58.1 & 60.0 & 55.7 & 57.4 & 12.1 & 17.4 & 15.0 & 19.5 & 14.1 & 19.2\\
\cline{3-15}
& & std & 3.1 & 2.9 & 6.1 & 6.0 & 6.4 & 6.2 & 2.8 & 1.5 & 3.1 & 2.1 & 3.2 & 2.0\\
\cline{2-15}
& \multirow{2}{*}{5} & mean & 68.6 & 70.5 & 128.5 & 129.9 & 84.9 & 86.8 & 25.8 & 28.9 & 38.9 & 41.3 & 36.1 & 38.9\\
\cline{3-15}
& & std & 4.1 & 3.8 & 7.9 & 7.7 & 9.9 & 9.6 & 4.7 & 4.4 & 9.4 & 9.2 & 7.5 & 7.1\\
\cline{2-15}
& \multirow{2}{*}{7} & mean & 95.6 & 97.4 & 180.3 & 181.7 & 102.7 & 104.7 & 44.2 & 46.7 & 74.6 & 76.6 & 49.5 & 52.7\\
\cline{3-15}
& & std & 4.9 & 4.7 & 10.1 & 9.9 & 5.5 & 5.3 & 6.5 & 6.3 & 13.2 & 13.1 & 6.8 & 6.7\\
\hline \hline

\multirow{6}{*}{300} 
& \multirow{2}{*}{3} & mean & 40.1 & 42.3 & 73.9 & 75.6 & 69.9 & 71.5 & 13.9 & 18.4 & 17.4 & 21.2 & 16.3 & 20.5\\
\cline{3-15}
& & std & 3.9 & 3.6 & 7.6 & 7.3 & 8.0 & 7.9 & 2.7 & 1.8 & 3.9 & 3.3 & 3.4 & 2.7\\
\cline{2-15}
& \multirow{2}{*}{5} & mean & 82.4 & 84.1 & 153.3 & 154.5 & 104.8 & 106.4 & 32.8 & 35.7 & 50.1 & 52.3 & 46.1 & 48.6\\
\cline{3-15}
& & std & 4.7 & 4.5 & 9.7 & 9.6 & 14.6 & 14.3 & 5.4 & 5.3 & 10.4 & 10.2 & 8.8 & 8.5\\
\cline{2-15}
& \multirow{2}{*}{7} & mean & 114.3 & 116.0 & 215.9 & 217.1 & 123.4 & 125.2 & 55.5 & 58.0 & 95.2 & 97.1 & 62.2 & 65.3\\
\cline{3-15}
& & std & 4.1 & 4.0 & 8.4 & 8.4 & 5.8 & 5.6 & 7.2 & 7.2 & 14.3 & 14.2 & 7.5 & 7.3\\
\hline

\end{tabular}

\caption{Widths (wd) and normalized widths (nwd) of various twin jointrees under rNET~\protect\cite{DarwicheECAI20b}. Refer to the main paper for the details of different jointree construction methods. All the thinned jointrees are constructed by bounding the functional chain lengths by $10$; see~\protect\cite{uai/ChenDarwiche22} for details.}
\label{tab:exp1}
\end{table*}

\begin{table*}[tb]
\scriptsize
\centering
\def\arraystretch{1.25}
\begin{tabular}{| *{9}{c|} | *{6}{c|}}
\hline
\multirow{2}{*}{\makecell{num\\ vars}} & \multirow{2}{*}{\makecell{max\\ pars}} & \multirow{2}{*}{stats} & \multicolumn{2}{c|}{\makecell{BASE-\\MF}} & \multicolumn{2}{c|}{\makecell{TWIN-\\ALG1}} & \multicolumn{2}{c||}{\makecell{TWIN-\\MF}} & \multicolumn{2}{c|}{\makecell{BASE-\\MF-RLS}} & \multicolumn{2}{c|}{\makecell{TWIN-\\THM3}} & \multicolumn{2}{c|}{\makecell{TWIN-\\MF-RLS}}\\
\cline{4-15}
 & & & wd & nwd & wd & nwd & wd & nwd & wd & nwd & wd & nwd &  wd & nwd \\
\hline

\multirow{6}{*}{50} 
& \multirow{2}{*}{3} & mean & 7.5 & 11.9 & 14.4 & 17.4 & 14.0 & 17.1 & 11.1 & 15.6 & 14.7 & 18.3 & 14.2 & 18.3\\
\cline{3-15}
& & std & 1.4 & 0.7 & 2.8 & 2.2 & 2.7 & 2.3 & 2.5 & 2.0 & 3.0 & 2.3 & 2.6 & 2.2\\
\cline{2-15}
& \multirow{2}{*}{5} & mean & 14.3 & 17.4 & 26.9 & 29.3 & 26.4 & 28.9 & 19.0 & 22.7 & 23.5 & 26.5 & 22.9 & 26.2\\
\cline{3-15}
& & std & 2.0 & 1.7 & 4.0 & 3.8 & 3.9 & 3.7 & 2.5 & 2.2 & 3.3 & 2.9 & 2.8 & 2.4\\
\cline{2-15}
& \multirow{2}{*}{7} & mean & 19.5 & 22.4 & 37.0 & 39.3 & 37.0 & 39.1 & 24.3 & 27.7 & 31.0 & 33.5 & 29.4 & 32.4\\
\cline{3-15}
& & std & 2.0 & 1.9 & 4.1 & 3.9 & 4.2 & 4.0 & 2.6 & 2.2 & 3.5 & 3.2 & 3.4 & 2.9\\
\hline \hline

\multirow{6}{*}{75} 
& \multirow{2}{*}{3} & mean & 10.2 & 14.0 & 18.9 & 21.6 & 18.5 & 21.3 & 14.9 & 18.5 & 19.1 & 22.1 & 18.1 & 21.5\\
\cline{3-15}
& & std & 1.9 & 1.3 & 3.6 & 3.4 & 3.6 & 3.4 & 3.2 & 2.8 & 4.1 & 3.6 & 3.7 & 3.2\\
\cline{2-15}
& \multirow{2}{*}{5} & mean & 21.3 & 23.9 & 39.1 & 41.1 & 38.5 & 40.7 & 26.9 & 29.7 & 33.5 & 35.8 & 31.8 & 34.5\\
\cline{3-15}
& & std & 2.3 & 2.2 & 5.1 & 4.9 & 5.1 & 4.9 & 2.7 & 2.5 & 3.8 & 3.6 & 3.7 & 3.3\\
\cline{2-15}
& \multirow{2}{*}{7} & mean & 28.8 & 31.4 & 54.1 & 56.2 & 53.9 & 55.9 & 34.5 & 36.9 & 44.9 & 46.9 & 42.3 & 44.5 \\
\cline{3-15}
& & std & 2.6 & 2.5 & 5.3 & 5.3 & 5.2 & 5.1 & 3.4 & 3.1 & 5.5 & 5.2 & 4.9 & 4.6\\
\hline \hline

\multirow{6}{*}{100} 
& \multirow{2}{*}{3} & mean & 13.7 & 16.8 & 25.1 & 27.5 & 24.4 & 26.9 & 19.5 & 22.4 & 24.9 & 27.1 & 23.1 & 25.9 \\
\cline{3-15}
& & std & 2.4 & 2.0 & 4.7 & 4.4 & 4.7 & 4.4 & 3.6 & 3.2 & 4.6 & 4.2 & 4.0 & 3.8\\
\cline{2-15}
& \multirow{2}{*}{5} & mean & 27.1 & 29.7 & 50.8 & 52.9 & 50.8 & 52.8 & 34.6 & 37.0 & 43.9 & 45.8 & 40.5 & 42.8 \\
\cline{3-15}
& & std & 2.5 & 2.5 & 5.4 & 5.2 & 5.1 & 5.0 & 3.9 & 3.7 & 5.4 & 5.2 & 4.8 & 4.5\\
\cline{2-15}
& \multirow{2}{*}{7} & mean & 38.7 & 40.9 & 73.0 & 74.8 & 72.8 & 74.5 & 45.8 & 47.8 & 61.3 & 62.9 & 57.4 & 59.4\\
\cline{3-15}
& & std & 3.5 & 3.4 & 6.7 & 6.5 & 6.8 & 6.7 & 4.7 & 4.5 & 7.9 & 7.7 & 7.4 & 7.1\\
\hline \hline

\multirow{6}{*}{125}
& \multirow{2}{*}{3} & mean & 16.7 & 19.6 & 30.8 & 33.0 & 30.5 & 32.7 & 23.8 & 26.5 & 29.5 & 31.8 & 28.0 & 30.6 \\
\cline{3-15}
& & std & 2.1 & 1.8 & 4.0 & 3.8 & 3.7 & 3.5 & 2.8 & 2.6 & 4.0 & 3.6 & 3.6 & 3.3\\
\cline{2-15}
& \multirow{2}{*}{5} & mean & 34.9 & 37.3 & 65.1 & 66.9 & 64.9 & 66.6 & 44.4 & 46.5 & 56.5 & 58.1 & 51.8 & 54.0\\
\cline{3-15}
& & std & 3.1 & 2.9 & 5.9 & 5.7 & 6.1 & 6.0 & 5.0 & 4.7 & 7.0 & 6.7 & 5.6 & 5.4\\
\cline{2-15}
& \multirow{2}{*}{7} & mean & 47.9 & 50.1 & 90.1 & 91.9 & 89.9 & 91.6 & 55.4 & 57.5 & 76.1 & 77.6 & 71.2 & 73.2\\
\cline{3-15}
& & std & 3.2 & 3.1 & 6.7 & 6.6 & 6.6 & 6.4 & 3.9 & 3.7 & 7.3 & 7.0 & 6.5 & 6.2\\
\hline \hline

\multirow{6}{*}{150} 
& \multirow{2}{*}{3} & mean & 20.3 & 22.9 & 36.9 & 38.9 & 36.2 & 38.2 & 27.3 & 30.0 & 34.0 & 36.1 & 31.9 & 34.3\\
\cline{3-15}
& & std & 2.4 & 2.3 & 4.7 & 4.7 & 4.4 & 4.3 & 3.5 & 3.3 & 4.4 & 4.2 & 4.5 & 4.2\\
\cline{2-15}
& \multirow{2}{*}{5} & mean & 41.4 & 43.6 & 76.9 & 78.6 & 76.3 & 77.9 & 50.6 & 52.5 & 65.3 & 66.8 & 60.6 & 62.5\\
\cline{3-15}
& & std & 3.3 & 3.1 & 7.4 & 7.2 & 6.9 & 6.9 & 4.3 & 4.2 & 7.3 & 7.1 & 6.5 & 6.2\\
\cline{2-15}
& \multirow{2}{*}{7} & mean & 56.7 & 58.7 & 106.7 & 108.2 & 106.0 & 107.5 & 65.7 & 67.4 & 91.1 & 92.4 & 83.2 & 85.0\\
\cline{3-15}
& & std & 4.1 & 3.8 & 8.1 & 7.9 & 7.5 & 7.4 & 4.5 & 4.4 & 8.0 & 7.9 & 7.0 & 6.9\\
\hline \hline

\multirow{6}{*}{200} 
& \multirow{2}{*}{3} & mean & 25.8 & 28.3 & 47.2 & 49.1 & 46.2 & 48.2 & 35.5 & 37.6 & 44.2 & 45.9 & 40.7 & 42.8\\
\cline{3-15}
& & std & 2.5 & 2.5 & 5.0 & 4.9 & 4.9 & 4.9 & 4.5 & 4.3 & 6.4 & 6.1 & 5.6 & 5.3\\
\cline{2-15}
& \multirow{2}{*}{5} & mean & 55.0 & 57.2 & 102.5 & 104.1 & 102.0 & 103.5 & 67.7 & 69.4 & 88.7 & 90.0 & 82.6 & 84.1\\
\cline{3-15}
& & std & 3.6 & 3.5 & 7.3 & 7.2 & 7.6 & 7.5 & 5.5 & 5.4 & 8.8 & 8.6 & 7.8 & 7.6\\
\cline{2-15}
& \multirow{2}{*}{7} & mean & 75.3 & 77.4 & 141.5 & 143.1 & 75.3 & 77.4 & 87.0 & 88.6 & 124.6 & 125.8 & 115.2 & 116.7\\
\cline{3-15}
& & std & 4.1 & 3.9 & 8.4 & 8.2 & 4.1 & 3.9 & 4.6 & 4.4 & 9.2 & 9.2 & 8.4 & 8.3\\
\hline \hline

\multirow{6}{*}{250} 
& \multirow{2}{*}{3} & mean & 31.6 & 34.0 & 58.1 & 60.0 & 57.5 & 59.4 & 43.8 & 45.8 & 54.3 & 55.9 & 49.8 & 51.9 \\
\cline{3-15}
& & std & 3.1 & 3.0 & 6.0 & 6.0 & 6.6 & 6.5 & 5.2 & 4.9 & 7.7 & 7.4 & 6.3 & 6.0\\
\cline{2-15}
& \multirow{2}{*}{5} & mean & 68.6 & 70.5 & 128.5 & 129.9 & 127.1 & 128.5 & 82.1 & 83.6 & 110.3 & 111.5 & 101.7 & 103.2\\
\cline{3-15}
& & std & 4.1 & 3.8 & 7.9 & 7.7 & 7.2 & 7.1 & 5.6 & 5.5 & 8.9 & 8.8 & 7.2 & 7.1\\
\cline{2-15}
& \multirow{2}{*}{7} & mean & 95.7 & 97.4 & 180.4 & 181.7 & 179.3 & 180.7 & 109.1 & 110.4 & 157.5 & 158.7 & 146.6 & 147.9\\
\cline{3-15}
& & std & 4.9 & 4.7 & 10.1 & 9.9 & 8.9 & 8.8 & 6.0 & 5.9 & 12.2 & 12.2 & 10.9 & 10.8\\
\hline \hline

\multirow{6}{*}{300} 
& \multirow{2}{*}{3} & mean & 40.1 & 42.3 & 73.9 & 75.6 & 72.4 & 74.1 & 54.1 & 55.8 & 66.7 & 68.1 & 61.8 & 63.6\\
\cline{3-15}
& & std & 3.9 & 3.6 & 7.6 & 7.3 & 7.5 & 7.4 & 5.1 & 5.1 & 6.8 & 6.7 & 7.1 & 7.0\\
\cline{2-15}
& \multirow{2}{*}{5} & mean & 82.5 & 84.1 & 153.3 & 154.5 & 152.7 & 154.0 & 99.1 & 100.4 & 135.0 & 136.2 & 124.3 & 125.8\\
\cline{3-15}
& & std & 4.7 & 4.5 & 9.7 & 9.6 & 10.4 & 10.3 & 6.3 & 6.2 & 10.3 & 10.2 & 9.9 & 9.8\\
\cline{2-15}
& \multirow{2}{*}{7} & mean & 114.4 & 116.0 & 215.9 & 217.2 & 215.1 & 216.3 & 131.5 & 132.7 & 194.1 & 195.2 & 178.8 & 180.1 \\
\cline{3-15}
& & std & 4.1 & 4.0 & 8.4 & 8.3 & 9.1 & 9.1 & 5.7 & 5.6 & 12.8 & 12.7 & 11.2 & 11.1\\
\hline

\end{tabular}

\caption{Widths (wd) and normalized widths (nwd) of various twin jointrees under rSCM~\protect\cite{DarwicheECAI20b}. Refer to the main paper for the details of different jointree construction methods. All the thinned jointrees are constructed by bounding the functional chain lengths by $10$; see~\protect\cite{uai/ChenDarwiche22} for details.}
\label{tab:exp2}
\end{table*}

\begin{table*}[tb]
\scriptsize
\centering
\def\arraystretch{1.25}
\begin{tabular}{| *{9}{c|} | *{6}{c|}}
\hline
\multirow{2}{*}{\makecell{num\\ vars}} & \multirow{2}{*}{\makecell{max\\ degree}} &
\multirow{2}{*}{stats} & \multicolumn{2}{c|}{\makecell{BASE-\\MF}} & \multicolumn{2}{c|}{\makecell{TWIN-\\ALG1}} & \multicolumn{2}{c||}{\makecell{TWIN-\\MF}} & \multicolumn{2}{c|}{\makecell{BASE-\\MF-RLS}} & \multicolumn{2}{c|}{\makecell{TWIN-\\THM3}} & \multicolumn{2}{c|}{\makecell{TWIN-\\MF-RLS}}\\
\cline{4-15}
 & & & wd & nwd & wd & nwd & wd & nwd & wd & nwd & wd & nwd &  wd & nwd \\
\hline

\multirow{6}{*}{50} 
& \multirow{2}{*}{5} & mean & 20.0 & 22.4 & 37.3 & 39.0 & 20.3 & 23.6 & 8.7 & 14.5 & 13.5 & 17.1 & 8.7 & 15.5\\
\cline{3-15}
& & std & 1.3 & 1.1 & 3.3 & 3.1 & 1.3 & 1.1 & 1.4 & 0.6 & 2.9 & 1.8 & 1.4 & 0.7\\
\cline{2-15}
& \multirow{2}{*}{10} & mean & 32.9 & 35.4 & 55.5 & 57.7 & 32.9 & 36.4 & 13.7 & 18.6 & 21.6 & 24.7 & 13.7 & 19.6\\
\cline{3-15}
& & std & 1.0 & 0.9 & 11.9 & 11.3 & 1.0 & 0.9 & 1.8 & 1.0 & 6.0 & 4.5 & 1.8 & 1.0\\
\cline{2-15}
& \multirow{2}{*}{15} & mean & 38.1 & 40.7 & 62.5 & 64.9 & 38.1 & 41.7 & 19.7 & 23.4 & 30.6 & 33.1 & 19.7 & 24.4\\
\cline{3-15}
& & std & 1.1 & 0.9 & 15.6 & 14.8 & 1.1 & 0.9 & 2.1 & 1.6 & 8.3 & 7.2 & 2.1 & 1.6\\
\hline \hline

\multirow{6}{*}{100} 
& \multirow{2}{*}{5} & mean & 40.0 & 41.8 & 76.7 & 78.1 & 40.1 & 42.7 & 14.5 & 18.4 & 23.5 & 25.8 & 14.6 & 19.4\\
\cline{3-15}
& & std & 1.2 & 1.4 & 6.1 & 5.9 & 1.7 & 1.5 & 2.1 & 1.6 & 5.0 & 4.6 & 2.1 & 1.6\\
\cline{2-15}
& \multirow{2}{*}{10} & mean & 65.1 & 67.4 & 110.6 & 112.7 & 65.5 & 68.6 & 33.0 & 35.9 & 52.3 & 54.5 & 33.0 & 36.9 \\
\cline{3-15}
& & std & 1.7 & 1.5 & 25.1 & 24.3 & 1.8 & 1.5 & 2.5 & 2.4 & 11.8 & 11.2 & 2.5 & 2.4\\
\cline{2-15}
& \multirow{2}{*}{15} & mean & 75.1 & 77.5 & 114.7 & 117.1 & 75.2 & 78.5 & 45.3 & 48.0 & 66.0 & 68.5 & 45.3 & 49.0 \\
\cline{3-15}
& & std & 1.8 & 1.6 & 33.0 & 32.2 & 1.8 & 1.6 & 3.1 & 3.1 & 18.7 & 17.7 & 3.1 & 3.1\\
\hline \hline

\multirow{6}{*}{150} 
& \multirow{2}{*}{5} & mean & 59.7 & 61.2 & 110.6 & 112.0 & 60.3 & 62.4 & 21.5 & 24.5 & 36.6 & 38.6 & 21.5 & 25.5 \\
\cline{3-15}
& & std & 2.6 & 2.3 & 17.2 & 16.8 & 2.5 & 2.3 & 2.5 & 2.3 & 7.8 & 7.4 & 2.5 & 2.3\\
\cline{2-15}
& \multirow{2}{*}{10} & mean & 97.1 & 99.0 & 159.6 & 161.7 & 97.3 & 100.1 & 49.8 & 52.5 & 80.9 & 83.1 & 49.8 & 53.5\\
\cline{3-15}
& & std & 2.0 & 1.8 & 37.0 & 36.7 & 2.0 & 1.8 & 3.6 & 3.6 & 20.1 & 19.3 & 3.6 & 3.6\\
\cline{2-15}
& \multirow{2}{*}{15} & mean & 109.7 & 111.8 & 168.9 & 171.3 & 109.8 & 112.8 & 66.7 & 69.1 & 92.7 & 95.3 & 72.3 & 75.6 \\
\cline{3-15}
& & std & 2.5 & 2.2 & 49.8 & 49.2 & 2.5 & 2.3 & 4.2 & 4.2 & 28.4 & 27.6 & 4.2 & 4.2\\
\hline \hline

\multirow{6}{*}{200} 
& \multirow{2}{*}{3} & mean & 109.7 & 111.8 & 168.9 & 171.3 & 109.8 & 112.8 & 66.7 & 69.1 & 92.7 & 95.3 & 72.3 & 75.6\\
\cline{3-15}
& & std & 3.0 & 2.8 & 23.3 & 22.9 & 2.6 & 2.5 & 3.3 & 3.3 & 10.0 & 9.6 & 3.2 & 3.2\\
\cline{2-15}
& \multirow{2}{*}{5} & mean & 127.0 & 128.5 & 193.6 & 195.8 & 127.2 & 129.6 & 65.0 & 67.4 & 97.2 & 99.6 & 65.0 & 68.4 \\
\cline{3-15}
& & std & 2.8 & 2.7 & 56.4 & 55.9 & 3.0 & 2.9 & 4.7 & 4.6 & 27.8 & 27.0 & 4.7 & 4.6\\
\cline{2-15}
& \multirow{2}{*}{7} & mean & 141.1 & 142.7 & 210.8	& 212.9	& 141.3 & 143.9 & 82.9 & 85.2 & 121.3 & 123.7 & 82.9 & 86.3\\
\cline{3-15}
& & std & 3.2 & 3.1 & 61.0 & 60.6 & 3.3 & 3.1 & 4.8 & 4.7 & 33.9 & 33.2 & 4.8 & 4.7\\
\hline

\end{tabular}

\caption{Widths (wd) and normalized widths (nwd) of various twin jointrees under rNET2~\protect\cite{sbia/IdeC02}. Refer to the main paper for the details of different jointree construction methods. All the thinned jointrees are constructed by bounding the functional chain lengths by $10$; see~\protect\cite{uai/ChenDarwiche22} for details.}
\label{tab:exp3}
\end{table*}

\begin{table*}[tb]
\scriptsize
\centering
\def\arraystretch{1.25}
\begin{tabular}{| *{9}{c|} | *{6}{c|}}
\hline
\multirow{2}{*}{\makecell{num\\ vars}} & \multirow{2}{*}{\makecell{max\\ degree}} &
\multirow{2}{*}{stats} & \multicolumn{2}{c|}{\makecell{BASE-\\MF}} & \multicolumn{2}{c|}{\makecell{TWIN-\\ALG1}} & \multicolumn{2}{c||}{\makecell{TWIN-\\MF}} & \multicolumn{2}{c|}{\makecell{BASE-\\MF-RLS}} & \multicolumn{2}{c|}{\makecell{TWIN-\\THM3}} & \multicolumn{2}{c|}{\makecell{TWIN-\\MF-RLS}}\\
\cline{4-15}
 & & & wd & nwd & wd & nwd & wd & nwd & wd & nwd & wd & nwd &  wd & nwd \\
\hline

\multirow{6}{*}{50} 
& \multirow{2}{*}{5} & mean & 20.0 & 22.4 & 36.2 & 39.9 & 37.2 & 39.2 & 27.9 & 30.2 & 39.6 & 41.3 & 37.1 & 39.1\\
\cline{3-15}
& & std & 1.3 & 1.1 & 2.4 & 2.3 & 2.7 & 2.4 & 2.0 & 2.0 & 3.4 & 3.1 & 2.9 & 2.6\\
\cline{2-15}
& \multirow{2}{*}{10} & mean & 32.9 & 35.4 & 64.8 & 66.7 & 64.9 & 66.8 & 37.4 & 39.8 & 54.0 & 55.8 & 50.6 & 52.9\\
\cline{3-15}
& & std & 1.0 & 0.9 & 2.4 & 2.3 & 2.5 & 2.5 & 1.5 & 1.2 & 3.1 & 2.9 & 2.3 & 2.3\\
\cline{2-15}
& \multirow{2}{*}{15} & mean & 38.1 & 40.7 & 75.3 & 77.4 & 76.4 & 78.5 & 40.8 & 43.4 & 61.1 & 63.2 & 60.3 & 62.5\\
\cline{3-15}
& & std & 1.1 & 0.9 & 2.4 & 2.3 & 3.4 & 3.1 & 1.3 & 1.2 & 3.5 & 3.4 & 3.4 & 3.4\\
\hline \hline

\multirow{6}{*}{100} 
& \multirow{2}{*}{5} & mean & 40.1 & 41.7 & 78.4 & 79.8 & 77.7 & 79.1 & 54.8 & 56.2 & 80.6 & 81.8 & 74.3 & 75.7\\
\cline{3-15}
& & std & 1.5 & 1.4 & 3.6 & 3.5 & 3.5 & 3.3 & 3.1 & 3.1 & 6.2 & 6.1 & 4.7 & 4.6\\
\cline{2-15}
& \multirow{2}{*}{10} & mean & 65.1 & 67.4 & 129.3 & 131.0 & 128.6 & 130.4 & 75.1 &	76.6 & 116.2 & 117.5 & 106.4 & 108.2 \\
\cline{3-15}
& & std & 1.7 & 1.5 & 3.6 & 3.4 & 3.9 & 3.6 & 2.5 & 2.3 & 5.6 & 5.4 & 4.1 & 4.0\\
\cline{2-15}
& \multirow{2}{*}{15} & mean & 75.1 & 77.5 & 149.6 & 151.5 & 149.3 & 151.2 & 81.2 & 83.1 & 131.0 & 132.5 & 125.7 & 127.8\\
\cline{3-15}
& & std & 1.8 & 1.6 & 3.8 & 3.5 & 3.4 & 3.3 & 2.2 & 2.0 & 5.0 & 4.9 & 4.9 & 4.8\\
\hline \hline

\multirow{6}{*}{150} 
& \multirow{2}{*}{5} & mean & 59.7 & 61.1 &	118.5 & 119.7 & 118.1 & 119.4 & 83.0 & 84.3 & 125.1 & 126.2 & 113.8 & 115.0\\
\cline{3-15}
& & std & 2.5 & 2.3 & 4.8 & 4.7 & 4.9 & 4.8 & 3.7 & 3.6 & 6.1 & 6.0 & 4.5 & 4.4\\
\cline{2-15}
& \multirow{2}{*}{10} & mean & 97.1 & 99.0 & 193.4 & 194.9 & 193.3 & 194.8 & 111.1 & 112.5 & 174.7 & 175.8 & 161.1 & 162.7\\
\cline{3-15}
& & std & 2.0 & 1.8 & 4.2 & 4.0 & 4.2 & 4.0 & 2.3 & 2.2 & 5.5 & 5.5 & 5.0 & 4.9\\
\cline{2-15}
& \multirow{2}{*}{15} & mean & 109.7 & 91.8 & 218.8 & 220.4 & 218.8 & 220.4 & 120.3 & 121.9 & 196.6 & 197.9 & 186.0 & 187.7\\
\cline{3-15}
& & std & 2.5 & 2.2 & 5.2 & 5.0 & 5.4 & 5.1 & 2.9 & 2.7 & 6.0 & 5.9 & 5.4 & 5.3\\
\hline \hline

\multirow{6}{*}{200} 
& \multirow{2}{*}{5} & mean & 75.6 & 76.6 & 150.7 & 151.8 & 150.8 & 151.9 & 102.5 & 103.6 & 155.0 &	156.1 &	143.3 & 144.4\\
\cline{3-15}
& & std & 2.8 & 2.6 & 5.6 & 5.5 & 5.3 & 5.2 & 4.9 & 4.8 & 7.2 & 7.1 & 5.4 & 5.4\\
\cline{2-15}
& \multirow{2}{*}{10} & mean & 127.0 & 128.5 & 253.4 & 254.6 & 252.7 & 254.0 & 149.4 & 148.6 & 233.0 & 234.2 & 212.9 & 214.3\\
\cline{3-15}
& & std & 2.8 & 2.7 & 5.5 & 5.5 & 4.5 & 4.4 & 3.4 & 3.4 & 7.7 & 7.7 & 6.2 & 6.1\\
\cline{2-15}
& \multirow{2}{*}{15} & mean & 141.1 & 142.7 & 281.7 & 283.0 & 281.5 & 282.8 & 156.3 & 159.6 & 255.6 & 256.7 & 238.6 & 240.2\\
\cline{3-15}
& & std & 3.2 & 3.1 & 6.4 & 6.2 & 6.5 & 6.4 & 4.3 & 4.2 & 9.7 & 9.7 & 7.0 & 7.0\\
\hline

\end{tabular}

\caption{Widths (wd) and normalized widths (nwd) of various twin jointrees under rSCM2~\protect\cite{sbia/IdeC02}. Refer to the main paper for the details of different jointree construction methods. All the thinned jointrees are constructed by bounding the functional chain lengths by $10$; see~\protect\cite{uai/ChenDarwiche22} for details.}
\label{tab:exp4}
\end{table*}

\end{document}